\documentclass{article}
\usepackage{ijcai21}

\usepackage{times}
\usepackage{soul}
\usepackage{url}
\usepackage[hidelinks]{hyperref}
\usepackage[utf8]{inputenc}
\usepackage[small]{caption}
\usepackage{graphicx}
\usepackage{amsmath}
\usepackage{amssymb}
\usepackage{amsthm}
\usepackage{dsfont}
\usepackage{booktabs}
\usepackage{algorithm}
\usepackage{algorithmic}
\usepackage{caption}
\usepackage{subcaption}
\usepackage{xcolor}
\newtheorem{lemma}{Lemma}

\title{Output Randomization: A Novel Defense for both White-box and Black-box Adversarial Models}

\author{
Daniel Park$^1$
\and
Haidar Khan$^2$\and
Azer Khan$^{3}$\and
Alex Gittens$^1$\And
B\"ulent Yener$^1$
\affiliations
$^1$Rensselaer Polytechnic Institute\\
$^2$Amazon Alexa\\
$^3$SUNY New Paltz\\
}

\begin{document}
\maketitle

\begin{abstract}
Adversarial examples pose a threat to deep neural network models in a variety of scenarios, from settings where the adversary has complete knowledge of the model in a "white box" setting and to the opposite in a "black box" setting. 
In this paper, we explore the use of output randomization as a defense against attacks in both the black box and white box models and propose two defenses. 
In the first defense, we propose output randomization at test time to thwart finite difference attacks in black box settings.
Since this type of attack relies on repeated queries to the model to estimate gradients, we investigate the use of randomization to thwart such adversaries from successfully creating adversarial examples.
We empirically show that this defense can limit the success rate of a black box adversary using the Zeroth Order Optimization attack to 0\%.
Secondly, we propose output randomization training as a defense against white box adversaries. 
Unlike prior approaches that use randomization, our defense does not require its use at test time, eliminating the Backward Pass Differentiable Approximation attack, which was shown to be effective against other randomization defenses.
Additionally, this defense has low overhead and is easily implemented, allowing it to be used together with other defenses across various model architectures.
We evaluate output randomization training against the Projected Gradient Descent attacker and show that the defense can reduce the PGD attack's success rate down to 12\%  when using cross-entropy loss.
\end{abstract}

\section{Introduction}
The success of deep neural networks has led to scrutiny of the security vulnerabilities in deep neural network based models. One particular area of concern is weakness to adversarial input: carefully crafted inputs that resist detection and can cause arbitrary errors in the model~\cite{Szegedy2013}. This is especially highlighted in the domain of image classification, where an attacker creates an image that resembles a natural image to a human observer but easily fools deep neural network based image classifiers~\cite{Kurakin2016}. 

Various attacks exists throughout the lifecycle of a deep neural network model. For example, adversaries can attack a model during training by injecting corrupting data into the training set or by modifying data in the training set. However, inference time attacks are more worrisome as they represent the bulk of realistic attack surfaces~\cite{Grosse2016,Narodytska2016,Carlini2018}. 

The input created under an inference time attack is known as an adversarial example and methods for generating such examples have recently attracted much attention. In most cases, the adversarial example is created by perturbing a (correctly classified) input such that the model commits an error. The perturbation applied to the "clean" input is typically constrained to be undetectably small~\cite{Papernot2016,Katholm1991,Carlini2017b}.

A defense against adversarial attacks is defined by the threat model it is designed to defend against~\cite{Carlini2019}. The most permissive threat model makes the weakest assumptions about the adversary. One such threat model can be assuming the adversary has complete knowledge of the model, including architecture and parameters of the underlying network. This is known as the "white box" setting. More restrictive threat models allow only so called “black box” attacks, attacks that can create adversarial examples without having access to the model architecture or weights and only accessing some form of output of the model.

In this paper, we mathematically and empirically explore the effect of output randomization on the robustness of a model against adversarial examples. Unlike existing randomization approaches that incorporate randomization by adding noise to the input, as in \cite{smilkov2017}, or to the output of convolutional layers, as in \cite{Liu2017}, we propose directly perturbing the output of a model. We examine the use of output randomization in two ways: 1) output randomization against finite difference based black box attacks and 2) output randomization training against white box attacks.

The rest of the paper is organized as follows. In Section \ref{sec:related} we cover related approaches. In Section~\ref{sec:threat}, we discuss the threat models considered. In Section~\ref{sec:attacks}, we describe the attacks utilized by the adversaries in our threat models. Output randomization as a black box and white box defense is described in Sections~\ref{sec:defense} and \ref{whitebox}, respectively. We show empirical results in Section~\ref{sec:results} and conclude in Section~\ref{sec:conclusion}.

\section{Related Work}
\label{sec:related}
\cite{Liu2017} proposes a random self-ensembling defense that injects noise into different layers of a target model with to protect against gradient-based attacks. To preserve performance, a prediction is ensembled over multiple queries. \cite{Lecuyer2019} take a similar approach in adding noise layers to their architecture, however also present guarantees on the defense's robustness against adversarial attack. 

Another popular method is to inject noise into the input pixels. \cite{smilkov2017} proposed adding noise to an image and averaging over multiple queries to achieve a sharper sensitivity map, which could be used to identify pixels that have a stronger influence over the predicted label. Many works have explored the idea of noisy inputs as a defense against adversarial examples. \cite{zantedeschi2017} proposed augmenting datasets with Gaussian noise. \cite{franceschi18} and \cite{fawzi2016} theoretically analyze input perturbation's effect on the robustness of target model. Other works have explored certified adversarial robustness using additive random noise \cite{Lecuyer2019,Li2018}

Although other defense methods consider introducing randomness to the input or model itself, this work is the first to our knowledge to consider randomizing the output of the model directly. Additionally, our proposed method does not inject randomization during the test phase to maintain high classification accuracy. This eliminates a class of attack, namely the Backward Pass Differential Approximation (BPDA) attacks introduced in \cite{Athalye2018}.

\section{Threat Model}\label{sec:threat}
For readability in this section, we will label our white box and black box models as $A_w$ and $A_b$, respectively.

The goal of both adversararies is to force the classifier to commit an error within a distortion limit, such that the example crafted by the adversary is similar to the original example. For $A_b$, we follow the threat model set by \cite{Chen2017} and use the $l_2$ perturbation penalty. Attackers under the $A_b$ model are allowed to query the model up to a maximum limit. For $A_w$, we follow the work of \cite{Madry2017} and use the $l_\infty$ norm. In this model, attackers use the PGD attack with either cross-entropy loss or the Carlini-Wagner loss, which optimizes the difference between the correct label's and highest incorrect label's assigned probabilities \cite{Carlini2017}. For the remainder of this paper, we will label the PGD attack with cross-entropy and Carlini-Wagner loss as PGD(Xent) and PGD(CW).

$A_b$ has access to the model only at the input and output level. $A_b$ is aware of the details of the defenses protecting the model and the type of randomness associated with any defense but not the exact random numbers generated. $A_w$ differs from $A_b$ only by having access to the entire model, including the model's architecture and weights.

\section{Preliminaries}
\label{sec:attacks}
In this section, we briefly go over finite difference estimation attacks and the projected gradient descent (PGD) attack. 

\subsection{Finite Difference Estimation Attacks}
Finite difference estimation attacks are called “gradient-free” since they do not involve computing gradients of the input by backpropagation on the target model. Instead, the gradients of the input are estimated by using the finite difference estimate for each input feature. 

Finite difference (FD) based approaches involve evaluating the adversarial loss at two points, $x+he_i$ and $x-he_i$, close to $x$ (with small $h$ and unit vector $e_i$) and using the slope to estimate the gradient of the loss with respect to pixel $i$ of the input. For a loss function $L$, the finite difference estimate of the gradient of pixel $i$ is given by:

\begin{equation}\label{eq:finitediff}
g_i = \frac{L(f(x+he_i)) - L(f(x-he_i))}{2h}
\end{equation}

Details specific to the ZOO and QL-attack can be found in the supplementary material \cite{Chen2017,Ilyas2018}.

\subsection{Projected Gradient Descent (PGD) attack}
A popular white box attack is the Fast Gradient Sign Method (FGSM), proposed in \cite{Goodfellow2014a}. Simply put, an adversary generates an adversarial example $x'$  from a natural image $x$ such that 
$$ x' = x + \epsilon \mbox{sign}(\nabla_x\mathcal{L}(\theta, x, y)) $$
where $\mathcal{L}$ is some loss function.

A more powerful variant of this gradient-based attack is to use Projected Gradient Descent, which iteratively updates $x$ along the direction of the gradient and projects it onto the natural input space \cite{Madry2017}. 

In addition to being used to evaluate defenses against $l_\infty$ attacks, \cite{Madry2017} showed the effectiveness of adversarial training with inputs generated using PGD.

\section{Thwarting Finite Difference Attacks}
\label{sec:defense}
The intuition behind output randomization is that a model may deliberately make errors in its predictions in order to thwart a potential attack. This idea introduces a tradeoff between accurate predictions and robustness against finite difference based black box attacks. 

Output randomization for a model that produces a probability distribution over class labels replaces the output of the model $p$ by a stochastic function $d(p)$ that must satisfy two conditions: 
\begin{enumerate}
    \item The probability of misclassifying an input due to applying $d$ is bounded by $K$
    \item The vector $d(p)$ prevents adversaries under the given threat model from generating adversarial examples.
\end{enumerate}
The first condition ensures that the applied defense minimally impacts honest users of the model. The defense's effectiveness comes from satisfying the second condition as the introduced randomness must prevent an adversary from producing an adversarial example.

In the following two sections, we consider a simple noise-inducing function $d(p) = p + \epsilon$ where $\epsilon$ is a random variable.

\subsection{Missclassification Rate}
A simple function useful for defending a model is the Gaussian noise function $d(p) = p + \epsilon$ where $\epsilon$ is a Gaussian random variable with mean $\mu$ and variance $\sigma^2$ ($\epsilon \sim \mathcal{N}(\mu,\sigma^2\cdot\textbf{I}_C)$). In the black box setting, a user querying the model with an input $x$ receives the perturbed vector $d(p)$ instead of the true probability vector $p$. Note that $d(p)$ does not necessarily represent a probability mass function like $p$. 

To verify that this function satisfies the first condition above, we wish to know the probability that the class predicted by the undefended model is the same as the class predicted by the defended model. If the output of the model for an input $x$ is $p$, we will refer to the maximum element of $p$ as $p_m$ and the rest of the elements of $p$ in decreasing order as $p_2, p_3, ... p_C$. 

Suppose the model correctly classifies the input $x$ in the vector $p$, we can express the probability that $x$ is misclassified in the vector $d(p)$ as less than or equal to the following sum:
\[ \sum_{i=2}^C{\mathbb{P}(d(p_i) > d(p_m)) } \]
We can write $ \mathbb{P}(d(p_i) > d(p_m)) $ for $i=2,3,...C$ as:
\[ \mathbb{P}(d(p_i) > d(p_m)) = \mathbb{P}(p_i + \epsilon_i > p_m + \epsilon_m) \]
If we define $\delta_i = p_m - p_i$, as shown in Figure~\ref{fig:deltas}, and since $e_i := \epsilon_i - \epsilon_m$ is itself a Gaussian with mean $\mu_i-\mu_m$ and variance $\sigma_i^2 + \sigma_m^2$ then we can write:
\begin{equation*}
\begin{aligned}
\mathbb{P}(e_i > \delta_i) & = 1 - \mathbb{P}(e_i \leq \delta_i) \\
                           & = 1 - \mathbb{P}\left(\frac{e_i - \mu_i + \mu_m}{\sigma_i^2 + \sigma_m^2} \leq \frac{\delta_i - \mu_i + \mu_m}{\sigma_i^2 + \sigma_m^2}\right)
\end{aligned}
\end{equation*}

Using the cumulative distribution function of a standard Gaussian distribution $\Phi$, we can write the misclassification probability 

\begin{equation*}
\begin{aligned}
K := \mathbb{P}(d(p_i) > d(p_m)) & = 1 - \Phi\left(\frac{\delta_i - \mu_i + \mu_m}{\sigma_i^2 + \sigma_m^2}\right) \\
                                 & = \Phi\left(-\frac{\delta_i - \mu_i + \mu_m}{\sigma_i^2 + \sigma_m^2}\right) 
\end{aligned}
\end{equation*}

For the special case of a Gaussian noise function $d(p)$ with mean 0 and variance $\sigma^2$ we would like to fix the probability of misclassification to a value $K$ and compute the appropriate variance $\sigma^2$. We can use the inverse of the standard Gaussian cdf $\Phi^{-1}$, or the probit function, to write this easily:

\[ \sigma^2 = -\frac{\delta_i}{2\Phi^{-1}(K)} \]

Note that the desired misclassification rate $K<0.5$ in any real case, so the rhs will be positive. If we consider $\delta_i$ as the confidence of the model, then the allowable variance will be larger when the model is confident and smaller otherwise. We show the calculations above for one class $i$. Note that the misclassification probability ($K$) and level of noise ($\sigma^2$) can be set for each class separately. In Figure~\ref{fig:vars} we show the maximum allowable variance for different misclassification rates.

\begin{figure}
    \centering
    \begin{subfigure}[t]{0.48\linewidth}
        \centering
        \includegraphics[width=.85\linewidth]{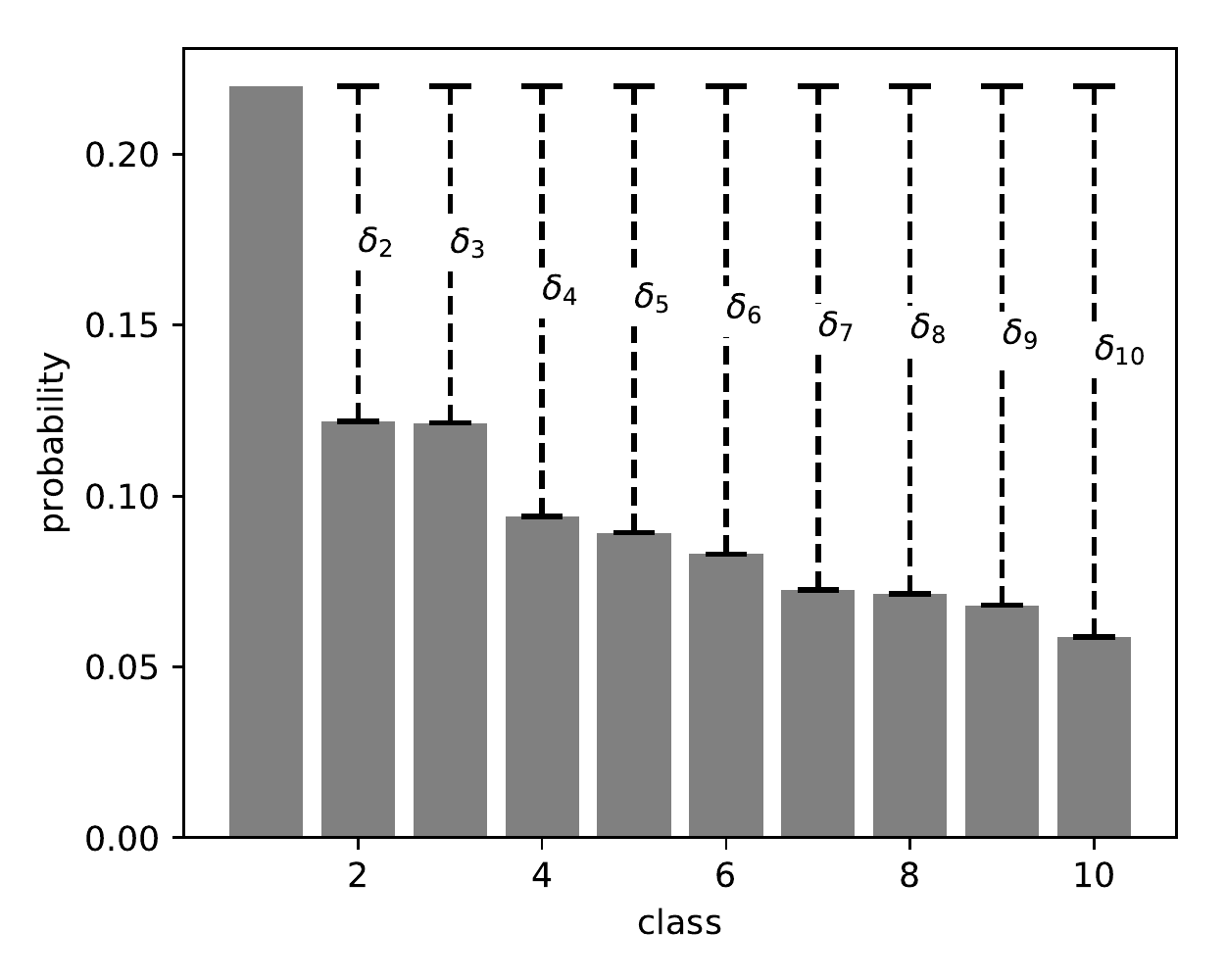}
        \caption{Probability distribution over classes generated by a classification model. $\delta_i$ represents the relative \mbox{confidence} of the model's prediction.}
        \label{fig:deltas}
    \end{subfigure}
    \hspace{.5mm}
    \begin{subfigure}[t]{0.48\linewidth}
        \centering
        \includegraphics[width=.85\linewidth]{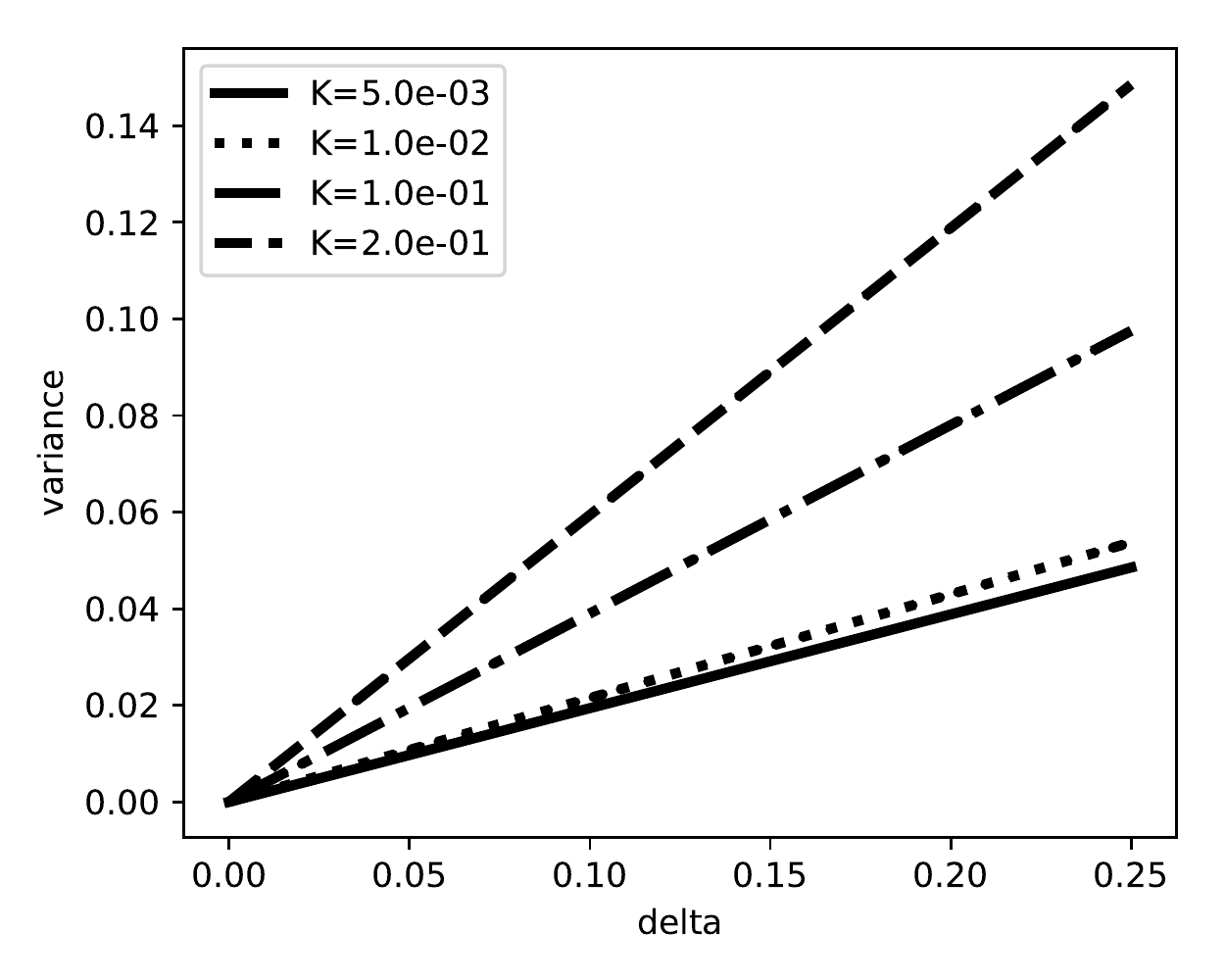}
        \caption{Maximum output randomization $\sigma^2$ vs $\delta_i$ for misclassification rates \mbox{$K= \{20\%, 10\%, 1\%, 0.5\%\}$}}
        \label{fig:vars}
    \end{subfigure}
    \caption{Controlling misclassification caused by output randomization}
\end{figure}

\subsection{Finite Difference Gradient Error}

To verify the function $d(p)$ satisfies the second condition we show the effect of the output randomization on the gradient accessible to the adversary.

Here, we write the adversarial loss function from Equation \ref{eq:finitediff} in terms of the output of the model to make explicit the dependence of $L$ on the output vector of the network $p = f(\cdot)$. $p$ and $p'$ are used to distinguish between the two output vectors needed to compute the gradient estimate. When the network is defended using output randomization, the function $d()$ is applied to the output vector of the network. Thus, the finite difference gradient computed by the adversary is:

\[ \gamma_i = \frac{L(d(p)) - L(d(p'))}{2h} \]

The error in the FD gradient introduced by the defense is given by $| g_i - \gamma_i|$. 
When $d$ is a function that adds noise $\epsilon$ to the output of the network, the expected value of the error is:
\[ | E[ g_i - \gamma_i]|  = \left| g_i - E\left[\frac{L(p+\epsilon) - L(p'+\epsilon')}{2h}\right]  \right| \]
Since the untargeted attack is generally considered easier than the targeted attack, we consider how the gradient error of the defended model behaves under the untargeted adversarial loss function. For untargeted attacks, we simplify the loss function to: $ L_u(p) = \log(p_c) - \log(p_o) = \log(\frac{p_c}{p_o})$ 
where $p_c$ is the probability of the true class and $p_o$ is the maximum probability assigned to a class other than the true class. 

Substituting the untargeted adversarial loss for $L(\cdot)$ we see:

\[ |E[ g_i - \gamma_i ]|  = \left| g_i - \frac{1}{2h} E\left[\log(\frac{p_c+\epsilon_c}{p_o+\epsilon_o}) - \log(\frac{p_c'+\epsilon_c'}{p_o'+\epsilon_o'})\right]  \right| \]

We use a second order Taylor series approximation of $E[\log(X)] \approx \log(E[X]) - \frac{Var[X]}{2E[X]^2}$ to approximate the expectations. If we further assume $\epsilon$ is zero-mean with variance $\sigma^2$, then $E[p+\epsilon] = p$ and the expectation of the defended gradient is approximately:

\[ |E[ g_i - \gamma_i ]|  \approx \left| \frac{\sigma^2}{4h}\left( \frac{\sigma^2 + p_o^2 + p_c^{2'}}{p_c^{2'} p_o^2} - \frac{\sigma^2 + p_o^{2'} + p_c^2}{p_c^2 p_o^{2'}} \right) \right| \]

This approximation summarizes the surprising effect output randomization has on finite difference based black box attacks. First, it is easy to see that the error scales with the variance of $\epsilon$ (in the zero mean case). Even when the adversary adapts to the defense by averaging over multiple queries the variance is only reduced linearly by the number of samples. Second, even in expectation the error is never non-zero. This is because one of two cases must be true for the error to be zero:
\begin{enumerate}
    \item $p_c == p_o$ and $p'_c == p'_o$
    \item $p_c == p'_c$ and $p_o == p'_o$
\end{enumerate}
Case 1 cannot occur because it implies the model is predicting two different classes simultaneously. Case 2 will only occur if $L(p) == L(p')$ which means $g_i = 0$. We show that this behavior holds on image classification datasets in Section \ref{empirical-results}.

\section{Training with Output Randomization}
\label{whitebox}
Learning using output randomization solves the noisy regularized empirical risk minimization (ERM) problem

\begin{equation}
\label{problem}
\min_\theta \mathbb{E}_{(x,y)\backsim \mathbb{P}} \mathbb{E}_{\epsilon} \mathcal{L}(f_{\theta}(x) + \epsilon, y),
\end{equation}

where $\mathcal{L}(\cdot, y)$ is a loss function that measures the discrepancy between the vector of predicted class probabilities (or logits) and $y$, the true one-hot encoding of the class of $x$. Here $\epsilon\sim \mathcal{N}(0, \Sigma)$ is independent Gaussian noise.

Intuitively, the higher the variance of $\epsilon$, the more robustly the model's prediction $f_\theta(x)$ must approximate $y$ to overcome the noise. In fact, we interpret output randomization as solving a regularized ERM that ensures the predicted class probability (or logit) vector is concentrated on the true class.  To do so, we first take the Taylor series expansion of $\mathcal{L}(\cdot, y)$ in Equation~\ref{problem}, centered around $f_\theta(x)$:

\begin{equation}
\label{eq-ts}
\begin{aligned}
\mathcal{L}(f_\theta(x) + \epsilon, y) =\ & \mathcal{L}(f_\theta(x), y) + \epsilon^\top\nabla_1\mathcal{L}(f_\theta(x), y) \\
                                  & + \epsilon^\top\nabla_1^2\mathcal{L}(f_\theta(x), y)\epsilon \\
                                  & + \mathcal{O}(||\epsilon||^3),
\end{aligned}
\end{equation}
where $\nabla_1$ ($\nabla_1^2$) denotes the gradient (Hessian) with respect to the first argument of $\mathcal{L}(\cdot, y)$.

Since $\epsilon \backsim \mathcal{N}(0, \Sigma)$, the expected value of the second and third terms of the r.h.s of Equation \ref{eq-ts} are as follows:

\begin{equation}
\label{eq1}
    \mathbb{E}[\epsilon^\top\nabla_1 \mathcal{L}(f_\theta(x), y)] = 0
\end{equation}

\begin{equation}
\label{eq2}
\begin{aligned}
    \mathbb{E}[\epsilon^\top\nabla_1^2 \mathcal{L}(f_\theta(x), y)\epsilon] & := \mathbb{E}[\epsilon^\top H\epsilon] \\
        & = \mathbb{E} \Big [\sum_{ij}\epsilon_i\epsilon_jH_{ij} \Big ] \\
        & = \sum_{ij} \Sigma_{ij} H_{ij} \\
        & = \mbox{Tr}(\Sigma \odot H),
\end{aligned}
\end{equation}
where $H = \nabla_1^2 \mathcal{L}(f_{\theta}(x), y)$ is the Hessian and $\odot$ denotes the Hadamard product. Equation \ref{eq-ts} can then be written as

$$ \mathbb{E}[\mathcal{L}(f_\theta(x), y)] + \mathbb{E}[\mbox{Tr}(\Sigma \odot H)] + \mathcal{O}(||\epsilon||^3) $$

Assuming that the higher-order terms of the Taylor series expansion can be ignored, solving the minimization problem from Equation \ref{problem} essentially solves

\begin{equation}
\label{rerm}
\begin{aligned}
    \theta = & \ \mbox{argmin}_\theta[\mathbb{E}_{(x,y)\backsim \mathbb{P}} \ \mathcal{L}(f_\theta(x), y) \\
             & + \mathbb{E}_{(x, y) \backsim \mathbb{P}} \mbox{Tr}(\Sigma \odot \nabla_1^2 \ \mathcal{L}(f_\theta(x), y)) ].
\end{aligned}
\end{equation}
This is a regularized ERM problem, where the regularization term is the expected trace of the Hadamard product of the noise covariance matrix and the Hessian of the loss. If $\Sigma = \sigma^2 \mbox{I}$ as in our experiments, then the regularizer simplifies to the expected sum of the eigenvalues of the Hessian:
\begin{equation*}
    \sigma^2 \cdot \mathbb{E}_{(x,y) \backsim \mathbb{P}}\ \sum_{i=1}^K \lambda_i
    \big(\nabla_1^2\mathcal{L}(f_\theta(x), y)\big).
\end{equation*}
Here $\lambda_i(\cdot)$ denotes the $i$th-largest eigenvalue of its argument. Thus we see that output randomization with standard Gaussian noise penalizes the expected curvature of the loss $\mathcal{L}(\cdot, y)$ around the model's predictions $f_\theta(x)$, and the variance of the noise serves as the regularization parameter. The reduced curvature implies that slight changes to the vector of predicted class probability (or logits) $f_\theta(x)$ will not significantly increase the loss. As a consequence, the margin between the true class and the other classes must be significant-- otherwise, the loss would significantly change by perturbing $f_\theta(x)$ by the small amount corresponding to the small margin.

For a more quantitative understanding of this phenomenon, take $f_{\theta}(x)$ to be the logits vector, as we do in our experiments, and $\mathcal{L}(v, y)$ to be cross-entropy loss between the one-hot vector $y$ and the categorical distribution with logits vector $v$. Then a single step of stochastic gradient descent with the output randomization objective $\mathcal{L}(f_\theta(x) + \epsilon, y)$ attempts to update the parameters $\theta$ in order to ensure that the coordinate in the perturbed logit vector $f_{\theta}(x) + \epsilon$ corresponding to the true class $i_{\text{true}}$ of $x$ is significantly larger than the other entries:
\begin{align*}
f_{\theta}(x)_{i_{\text{true}}} & \gg \left( \max\limits_{i \neq i_{\text{true}}} f_{\theta}(x)_i + \epsilon_i \right) - \epsilon_{i_{\text{true}}} \\
 & > f_{\theta}(x)_{i^\star} + (\epsilon_{i^\star} - \epsilon_{i_{\text{true}}}),
\end{align*}

where $i^\star$ denotes the index of the largest entry in the noise vector $\epsilon$. 
Then $i^\star$ is uniformly distributed over the integers $1,\ldots,K$, and it is shown in the supplement that $\epsilon_{i^\star} = \Omega(\sigma \sqrt{\ln(K)})$ with high probability. 
Similarly, with high probability $\epsilon_{i_{\text{true}}} = O(\sigma)$. 
We see therefore that each step of sgd with output randomization attempts to update the parameters $\theta$ to ensure that the logit of the true class is larger than the logit of a uniformly randomly chosen class by an additive factor on the order of $\Omega(\sigma\sqrt{\ln(K)})$:

\[
f_{\theta}(x)_{i_{\text{true}}} > f_{\theta}(x)_{i^\star} + \Omega(\sigma\sqrt{\ln(K)}).
\]
This margin can be increased by increasing the parameter $\sigma$. Thus we see that output randomization (applied to the logits, and using cross-entropy loss) attempts to impose a margin on average, in the sense that in each step of sgd it randomly chooses an incorrect class and attempts to update the model so that the logit of the correct class is significantly larger than the logit of the selected incorrect class.

\section{Empirical Results} \label{empirical-results}
\label{sec:results}
We evaluate with a GeForce RTX 2070 GPU (12GB), a 4-core Intel(R) Core(TM) i7-7700K CPU @ 1.15GHz, and 32GB of memory~\footnote{Our code will be made available upon publication.}. We discuss our experiments with black box finite difference attacks in Section \ref{blackresults} and our experiments with white box PGD attacks in Section \ref{whiteresults}.

\subsection{Black Box: Finite Difference Attacks} \label{blackresults}
To evaluate output randomization against finite difference attacks, we select three attacks (ZOO~\cite{Chen2017}, QL~\cite{Ilyas2018}, and BAND~\cite{Ilyas2018bandit}) on benchmark image classification datasets (MNIST~\cite{lecun2010mnist}, CIFAR10~\cite{krizhevsky2014cifar}, and ImageNet~\cite{deng2009imagenet}). Keeping with \cite{Chen2017}, we keep $h=0.0001$. For all defended models, we use $\epsilon \sim \mathcal{N}(0, \sigma^2 \cdot \mathbf{I}_C)$. We also adapt the attacks by allowing the adversary to average over the output randomization in an attempt to bypass the defense.  

We trained models for MNIST, CIFAR10, and ImageNet that achieved 99\%, 79\%, and $ {\sim} $ 72\% test set accuracies, respectively, following \cite{Chen2017}. Non-adaptive attacks were conducted using the parameters suggested by the attacks. The adaptive attackers (i) can average over multiple queries and (ii) had their maximum query limit was doubled. For all of our experiments, we averaged the attack success rate over 100 images and report the mean value over 30 runs.

    \begin{figure}[h]
        \includegraphics[width=\linewidth]{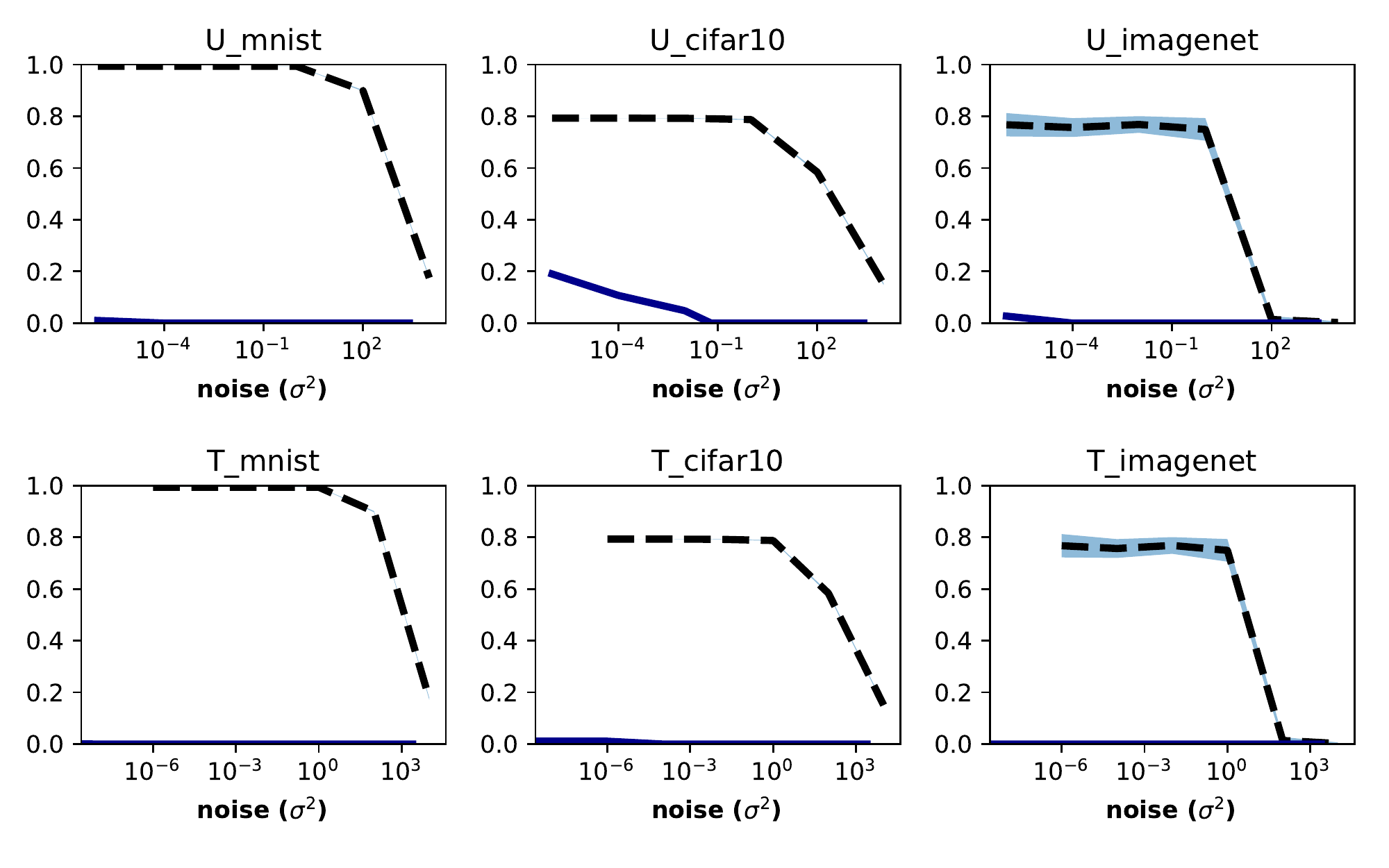}
        \caption{ZOO success rate (solid) and test set accuracy (dashed) vs variance for the non-adaptive attacker. Top row shows untargeted and bottom row shows targeted attacks. Output randomization blocks attacks even at very small noise levels ($\sigma^2 < 1\text{e-}6$).}
        \label{fig:black_asr_noise}
    \end{figure}

Our main set of experiments is shown in Figures \ref{fig:black_asr_noise} 
and show the effects of output randomization on the non-adaptive and adaptive ZOO attacks. Our experimentation with the adaptive attack showed that the attacker is not able to successfully mount an attack. We have included a figure in the supplemental material and continue discussion below. We show that the defense reduces the attack success rate significantly even in the adaptive attacker setting. The effectiveness of the defense was not affected by targeted or untargeted attacks. Table~\ref{tab:bb_imagenet} summarizes the results for three finite difference based black box attacks on ImageNet. 

\begin{table}[ht]
  \caption{Output randomization vs 3 black box attacks on 100 correctly classified ImageNet examples measured by attack success rate.}
  \label{tab:bb_imagenet}
  \centering
  \begin{tabular}{lccc}
    Variance $\sigma^2$ & ZOO & QL & BAND\\
    \midrule
    (undefended) & 0.69 &  1.00 & 0.92 \\
    1.00\text{e-}4 & 0.03 &  0.73 & 0.58 \\
    1.00\text{e-}2 & 0.00 & 0.02 & 0.07 \\
    5.76\text{e-}2 & 0.00 & 0.01 & 0.06 \\
    \bottomrule
  \end{tabular}
\end{table}

Figure \ref{fig:black_asr_noise} shows the effect of increasing noise levels on test set accuracy. Output randomization is an effective defense against black box attacks at noise levels as small as $\sigma^2=1\text{e-}4$ where model performance is identical to undefended models.

\subsubsection{Using the Newton Solver for the ZOO Attack}
\cite{Chen2017} suggests using the Adam solver for their proposed ZOO attack for its efficiency over the Newton solver. However, we believe that the lower average $l_2$ perturbation is the cause of low success rates against output randomization even when using the adaptive approach. 

\begin{figure}[h]
    \centering
    \includegraphics[width=\linewidth]{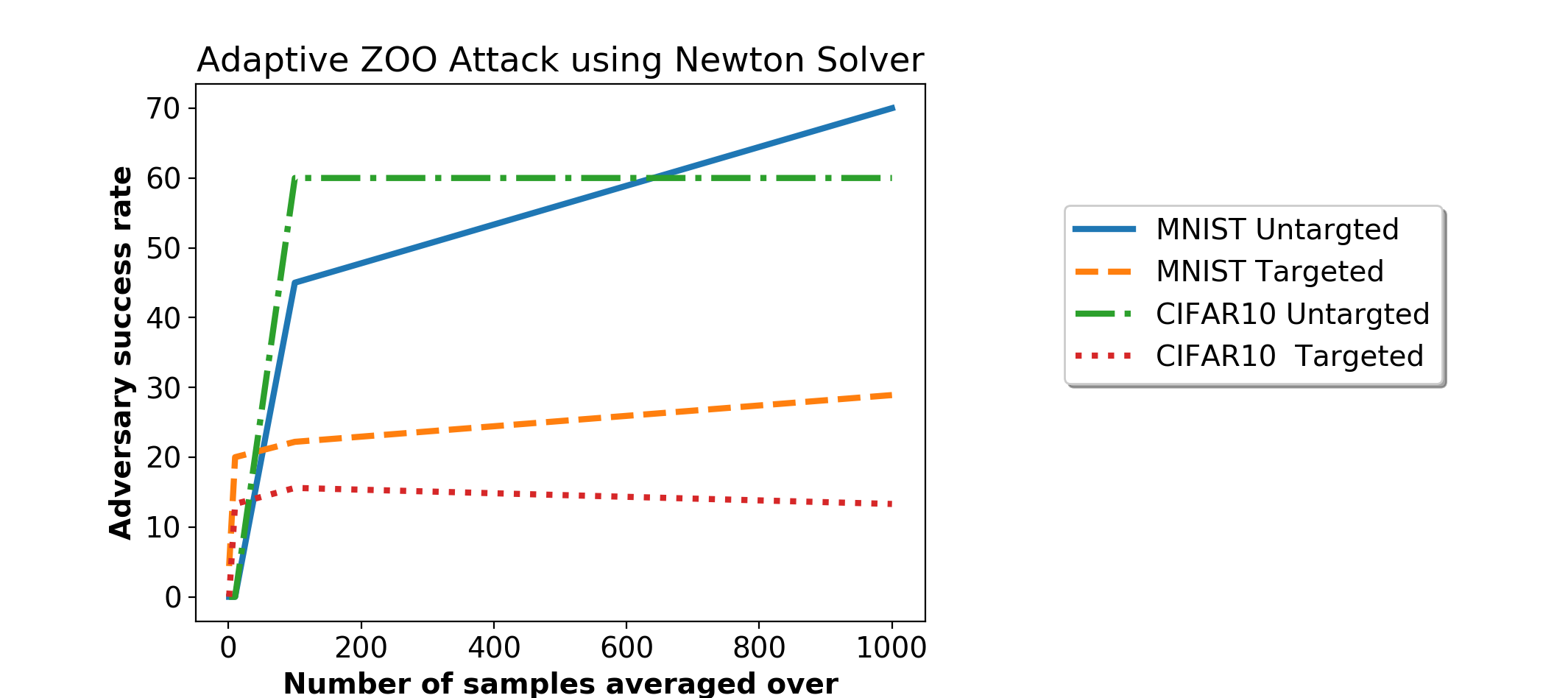}
    \caption{Results of untargeted and targeted adaptive ZOO attack against MNIST and CIFAR10 trained models using the output randomization defense with a $\sigma=1$.}
    \label{fig:newton}
\end{figure}

\begin{table*}[ht]
  \caption{Accuracy and robustness of Wide Resnet 32-10 and its non-wide variant against the PGD(Xent) and PGD(CW) attacks. We present the results using no defense, PGD adversarial training, output randomization training with $\sigma=100$, and a combination of the two defenses.}
  \label{tab:wres}
  \centering
  \begin{tabular}{llccc}
    Model &Defense & Clean Accuracy & PGD (Xent) & PGD (CW) \\
    \midrule
    Wide Resnet 32-10 &(undefended) & 95.32\% &  0.0\%  & 0.0\% \\
    &PGD Trained \cite{Madry2017} & 86.30\% &  47.95\% & 48.61\% \\
    &OR[stddev=100] Trained & 94.38\% & 88.07\% & 0.0\% \\
    &PGD+OR[stddev=100] Trained & 88.92\% & 82.03\% & 42.09\% \\
    \midrule
    Resnet 32 &(undefended) & 92.77\% &  0.0\%  & 0.0\% \\
    &PGD Trained \cite{Madry2017} & 77.87\% &  45.37\% & 44.27\% \\
    &OR[stddev=100] Trained & 91.51\% & 84.99\% & 0.0\% \\
    &PGD+OR[stddev=100] Trained & 81.98\% & 62.63\% & 39.62\% \\    
    \bottomrule
  \end{tabular}
\end{table*}

Unlike the Adam solver, the Newton solver allows the adaptive adversary to successfully attack the target model by averaging over $k$ samples per query, as shown in Figure \ref{fig:newton}. We see similar behavior when decreasing the value of $\sigma$ for the defense, however, we report the results using $\sigma=1$ as we experimentally found it to be the lowest value for which the non-adaptive adversary is unsuccessful. For higher values of $\sigma$, the adversary is able to mount attacks, however is unable to match the success rates shown here. We believe this is due to our computational limitations on allowing a higher number of queries.

\subsection{White Box: PGD attack} \label{whiteresults}
To evaluate the effectiveness of "training with output randomization" as a defense against white-box attacks, we compare our method against Projected Gradient Descent (PGD) adversarial training as described in \cite{Madry2017} using the CIFAR10 dataset.
We follow the experimental parameters set in \cite{Madry2017}.
Our PGD attacks use $\epsilon=0.8$, step size of 2 with 10 total steps, and evaluate on 10000 evaluation examples.
Additionally, the defenses are applied to Wide Resnet 32-10 (\textit{WResnet}) and Resnet 32 (\textit{Resnet}), which are trained for 80000 iterations.
We report the results of output randomization training sampling noise from a zero-mean Gaussian with $\sigma=100$.
We show the results of different $\sigma$ values in the supplementary material.

As a baseline, we trained both \textit{WResnet} and \textit{Resnet} on the CIFAR10 dataset. For each architecture, we trained additional models using PGD adversarial training, output randomization training, and both defenses.
We evaluated each model against the same white-box adversary and report the results in Table \ref{tab:wres}. 

We found that output randomization trained models achieved higher accuracies on the clean test set and on adversarial examples created using PGD(Xent) compared to PGD adversarially trained models. For \textit{WResnet}, we achieve a 12.08\% and 41.92\% improvement in clean test set accuracy and PGD(Xent) adversarial examples, respectively, over PGD adversarial training. For \textit{Resnet}, we see similar behavior with improvements of 13.64\% and 39.62\% over PGD adversarial training. We also found that our defense is weak against PGD(CW) attacks. This is further discussed in Section \ref{cwresults}

Output randomization training increases the models' robustness against PGD(Xent) attacks with high effectiveness and $<$1\% accuracy-robustness trade-off without requiring randomization at test time.

\subsubsection{Carlini-Wagner Loss Function} \label{cwresults}
We found the PGD(CW) to have a $>$99\% success rate against output randomized trained model. We believe that the PGD(Xent) attacker gets stuck in a local maximum. However, the PGD(CW) attacker is able to overcome this because it is additionally directed by the most probable incorrect label. In Figure \ref{fig:loss}, we show the loss over each step of an example PGD(CW) attack. We collect the loss over 40 steps to achieve a better understanding of how the loss behaves over time.

\begin{figure}[h]
    \centering
    \includegraphics[width=0.85\linewidth]{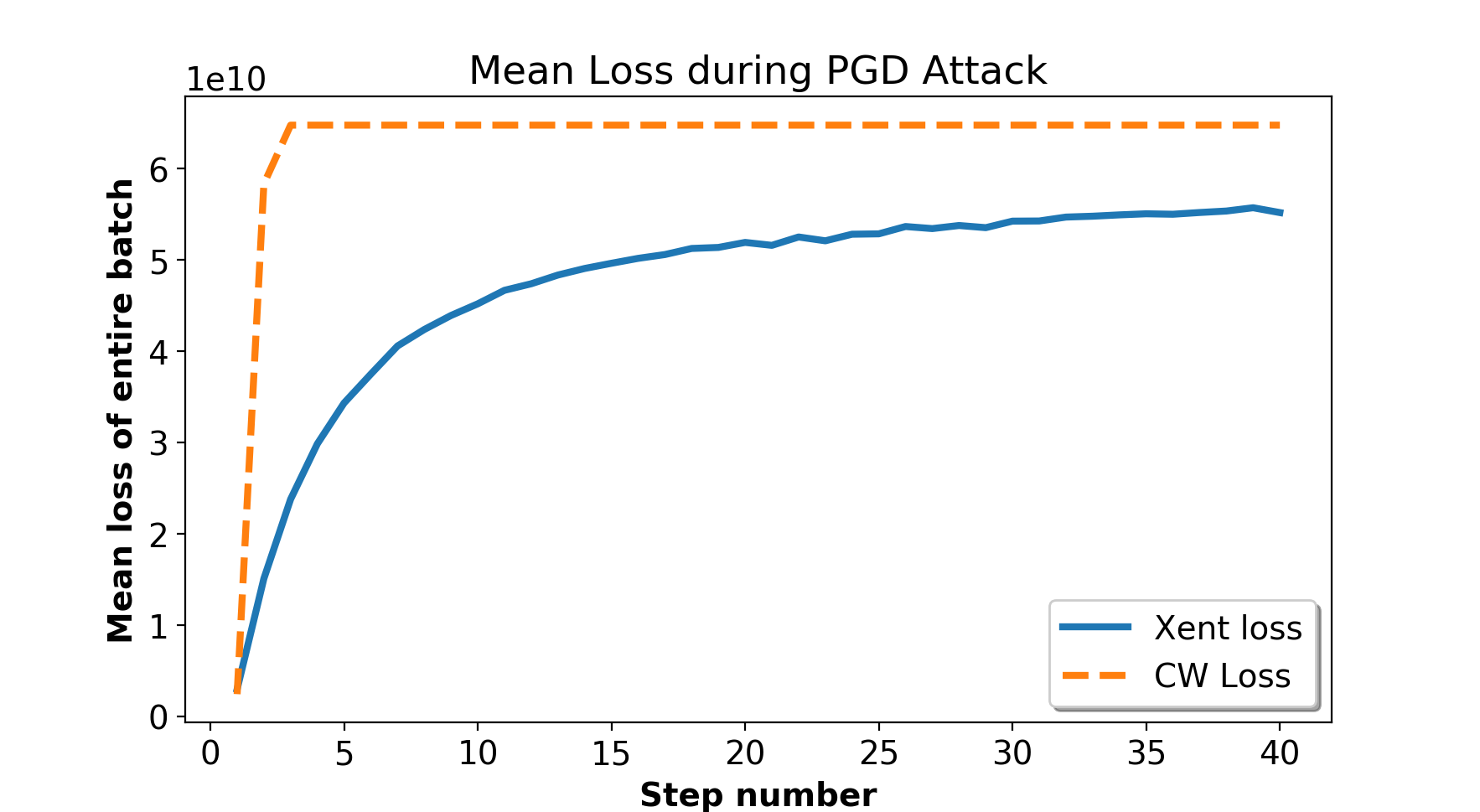}
    \caption{The mean loss scaled by 1e10 during example PGD(Xent) and PGD(CW) attacks over 40 steps against Wide Resnet 32-10 trained with output randomization ($\sigma=100$).}
    \label{fig:loss}
\end{figure}

The PGD(Xent) loss gradually increases before asymptotically approaching ${\sim}$5e10. However, the PGD(CW) loss rapidly increases to ${\sim}$6.5e10. This leads us to believe that our prior hypothesis is correct and the PGD(Xent) attacker is caught in a local maximum.

Due to the simplicity of output randomization, it can used in conjunction with PGD adversarial training. For both models, we see that defending a model with both adversarial training and output randomization training helped offset weaknesses from each defense. Compared to PGD adversarial training, using both output randomization and PGD adversarial training increased the clean test set accuracy of \textit{WResnet} and \textit{Resnet} from 86.30\% to 88.92\% and from 77.87\% to 81.98\%, respectively. Additionally, we see that the models' ability to correctly classify PGD(Xent) attacks increased by ${\sim}17\%$ and ${\sim}35\%$. Lastly, we see that PGD adversarial training offsets output randomization training's weakness against PGD(CW) attacks.

\section{Conclusion}
\label{sec:conclusion}
In this paper, we present two ways output randomization can be used to harden deep neural networks against adversarial attacks. By introducing simple randomization to the output of the model,  we show that finite different gradient estimate attacks can be thwarted. We evaluated output randomization against three successful attacks on three benchmark datasets. Additionally, we show that training models with output randomization increases their robustness against PGD attacks. We empirically show that Wide Resnet 32-10 achieves up to 88.07\% accuracy on PGD adversarial examples with only a 1\% decrease in clean test set accuracy. Furthermore, due to the simple nature of the presented defense, output randomization training can be used together with other defenses.

\bibliographystyle{named}
\bibliography{ijcai21}

\clearpage








\section{Appendix}

\subsection{ZOO black box attack}
The Zeroth Order Optimization based black box (ZOO) attack~\cite{Chen2017} is a method for creating adversarial examples that only requires input and output access to the model. ZOO adopts an iterative optimization based approach, similar to the Carlini \& Wagner (C\&W) attack~\cite{Carlini2017}. The attack begins with a correctly classified input image $x$, defines an adversarial loss function that scores perturbations $\delta$ applied to the input, and optimizes the adversarial loss function using gradient descent to find $\delta^*$ that creates an adversarial example. 

The primary adversarial loss used by the ZOO attack for targeted attacks is given by:
\begin{equation}
    L(x,t) = \max{\left\{ \max_{i\neq t}{\left\{ \log{f(x)}_i - \log{f(x)}_t\right\}}, -\kappa \right\} }
    \label{targeted}
\end{equation}
Where $x$ is an input image, $t$ is a target class, and $\kappa$ is a tuning parameter. Minimizing this loss function over the input $x$ causes the classifier to predict class $t$ for the optimized input. For untargeted attacks, a similar loss function is used:
\begin{equation}
    L(x) = \max{\left\{ \log{f(x)}_i - \max_{j\neq i}{\left\{ \log{f(x)}_j\right\}}, -\kappa \right\} }
    \label{untargeted}
\end{equation}
where $i$ is the original label for the input $x$. This loss function simply pushes $x$ to enter a region of misclassification for the classifier $f$. 

In order to limit distortion of the original input, the adversarial loss function is combined with a distortion penalty in the full optimization problem. This is  given by:

\[
    \min_{x}{\Vert x - x_0 \Vert_2^2 + c\cdot L(x,t) };\ \ \text{subject to } x \in [0,1]^n
\]

ZOO uses "zeroth order stochastic coordinate descent" to optimize input on the adversarial loss directly. This is most easily understood as a finite difference estimate of the gradient of the input with the symmetric difference quotient~\cite{Chen2017}:
\[
\frac{\delta L}{\delta x_i} \approx g_i := \frac{L(x+he_i) - L(x-he_i)}{2h}  
\]
with $e_i$ as the basis vector for coordinate/pixel $i$ and step size $h$ set to a small constant. ZOO uses this approximation of the gradients to create an adversarial example from the given input.  We keep $h=.0001$ in all experiments to stay consistent with the original experimental setup. However, in Section \ref{empirical-results-appendix}, we present result with varying $h$ values.

\subsection{Query Limited (QL) black box attack}
A similar approach to ZOO is adopted by \cite{Ilyas2018} in a query limited setting. Like ZOO, the QL attack estimates the gradients of the adversarial loss using a finite difference based approach. However, the QL attack reduces the number of queries required to estimate the gradients by employing a search distribution. Natural Evolutionary Strategies (NES)~\cite{Wierstra2014natural} is used as a black box to estimate gradients from a limited number of model evaluations. Projected Gradient Descent (PGD)~\cite{Madry2017} is used to update the adversarial example using the estimated gradients. PGD uses the sign of the estimated gradients to perform an update:
$x^t = x^{t-1} - \eta \cdot \text{sign}(g_t)$, with a step size $\eta$ and the estimated gradient $g_t$. The estimated gradient for the QL attack using NES is given by:
\[
g_t = \sum_{i=1}^m{\frac{L(x + \sigma \cdot u_i)\cdot u_i - L(x - \sigma \cdot u_i) \cdot u_i}{2m\sigma}}
\]
where $u_i$ is sampled from a standard normal distribution with the same dimension as the input $x$, $\sigma$ is the search variance, and $m$ is the number of samples used to estimate the gradient.

\subsection{Regularizing effect of Output Randomization}
Proofs are provided for the two claims about Gaussian random variables used in section 6 of the main paper. Both claims follow from classical estimates on the Mill's ratio of a standard Gaussian: if $\Phi^c(t) := \frac{1}{\sqrt{2\pi}} \int_t^\infty \exp(\tfrac{-x^2}{2})\, dx$ denotes the complementary error function, then
\begin{multline}
\label{eqn:mills}
\frac{1}{\sqrt{2\pi}} \cdot \frac{1}{t + 1} \exp\left(\frac{-t^2}{2}\right) \geq \Phi^c(t) \\
\geq \frac{1}{\sqrt{2\pi}} \cdot \frac{1}{t + \sqrt{t^2+2}} \exp\left(\frac{-t^2}{2}\right)
\end{multline}
when $t$ is nonnegative~\cite{NIST:DLMF}.

The first claim is that a single draw of a Gaussian is $O(\sigma)$ (or negative), with high probability. 
\begin{lemma} Consider $g \sim \mathcal{N}(0, \sigma^2)$. For any probability $\delta \in (0,1)$, it holds that
$$\mathbb{P}\left(g < \sqrt{2 \ln\left(\delta^{-1}\right)}\sigma\right) \geq 1 - \delta.$$
\end{lemma}

\begin{proof}
Let $\tilde{g} = g/\sigma$, so that $\tilde{g}$ is a standard Gaussian. By Equation~\ref{eqn:mills}, for any constant $C > 0$,
\begin{align*}
 \mathbb{P}(g < C \sigma) & = \mathbb{P}(\tilde{g} < C) = 1 - \Phi^c(C) \\
 & \geq 1 - \frac{1}{\sqrt{2\pi}} \cdot \frac{1}{C + 1} \exp(-C^2/2) \\
 & \geq 1 - \frac{1}{\sqrt{2\pi}} \exp(-C^2/2).
\end{align*}
The claim follows by noting that $\sqrt{2 \ln(\delta^{-1})}$ is positive and assigning this value to $C$ ensures that $\exp(-C^2/2) = \delta$.
\end{proof}
The second claim is that the largest of $K$ i.i.d. draws of a Gaussian is $\Omega(\sigma \sqrt{\ln(2K)})$ with high probability.
\begin{lemma} Consider $X \sim \mathcal{N}(0, \sigma^2 I_K )$. Fix a $D > 0$ and set $\nu = \ln(2 D \ln(2K))$, then 
$$\mathbb{P}\left( \max\limits_{i=1,\ldots,K} X_i < \sigma \sqrt{2 \ln(2K) - \nu} \right) \leq \exp\left( - \frac{\sqrt{D}}{8 \sqrt{2\pi}} \right).$$
\end{lemma}

\begin{proof}
For convenience, define $q = \sqrt{2\ln(2K) - \nu}$ and let $\tilde{X} = \frac{1}{\sigma}X$. This gives that $X \sim \mathcal{N}(0, I_k)$ and
\begin{align*}
    \mathbb{P}(\max\limits_{i=1,\ldots,K} X_i < \sigma\sqrt{2 \ln(2 K) - \delta)} & \leq \mathbb{P}(\max\limits_{i=1,\ldots,K} \tilde{X}_i < q).
\end{align*}
By Equation~\ref{eqn:mills},
\begin{align*}
\mathbb{P}(\max\limits_{i=1,\ldots,K} \tilde{X}_i < q) & = \mathbb{P}(\tilde{X}_1 < q)^K = (1 - \Phi^c(q))^K \\
& \mkern-54mu\leq \left(1 - \frac{1}{\sqrt{2\pi}} \cdot \frac{1}{q + \sqrt{q^2+2}} \exp\left(\frac{-q^2}{2}\right) \right)^K  \\
& \mkern-54mu \leq \left(1 - \frac{1}{2\sqrt{2\pi}} \cdot \frac{1}{q + 1} \exp\left(\frac{-q^2}{2}\right)\right)^K.
\end{align*}
The last inequality holds because the subadditivity of the square root function implies $q + \sqrt{q^2 + 2} \leq 2(q + 1).$ 

Using the fact that $1 + x \leq \exp(-x)$ for any $x$, we see that
\begin{align*}
\mathbb{P}(\max\limits_{i=1,\ldots,K} \tilde{X}_i < q) & \leq
 \exp\left(-\frac{1}{2\sqrt{2\pi}} \cdot \frac{K}{q + 1} \exp\left(\frac{-q^2}{2}\right) \right) \\
 & \mkern-16mu = \exp\left(-\frac{1}{4\sqrt{2\pi}} \cdot \frac{\exp(\nu/2)}{q + 1}\right) \\
 & \mkern-16mu= \exp\left(-\frac{1}{4\sqrt{2\pi}} \cdot \frac{\sqrt{2D\ln(2K)}}{\sqrt{2\ln(2K) - \nu} + 1}\right) \\
 & \mkern-16mu\leq \exp\left(-\frac{1}{4} \cdot \sqrt{\frac{D}{2\pi}} \cdot \frac{\sqrt{2\ln(2K)}}{\sqrt{2\ln(2K)} + 1}\right) \\
 & \mkern-16mu\leq \exp\left(-\frac{1}{8} \cdot \sqrt{\frac{D}{2\pi}} \right).
\end{align*}
In the last inequality, we used the fact that $K \geq 1$, which implies that $\sqrt{2 \ln(2K)}/(\sqrt{2 \ln(2K)} + 1) \geq \tfrac{1}{2}.$
\end{proof}

\subsection{Empirical results: Black box finite difference} \label{empirical-results-appendix}
In this section, we present supplemental empirical results in our evaluation of output randomization as a defense against black box finite difference attacks.

\subsubsection{Sanity check}
    \begin{figure}[h]
        \centering
        \includegraphics[width=\linewidth]{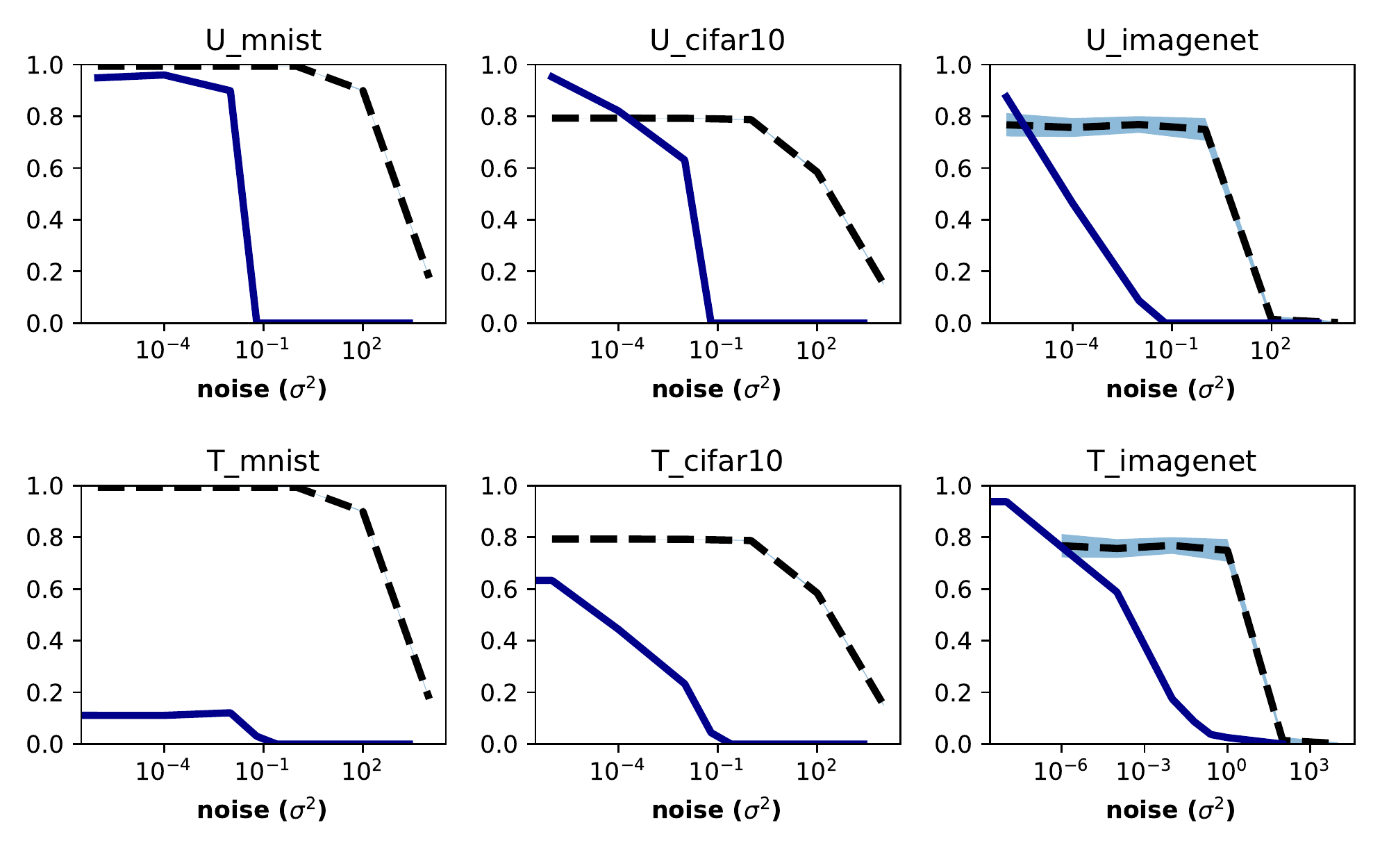}
        \caption{CW Attack success rate (solid) and test set accuracy (dashed) vs variance for the non-adaptive attacker. Top row shows untargeted and bottom row shows targeted attacks. Output randomization is not effective against a white box attacker at small noise levels.}
        \label{fig:white_asr_noise}
    \end{figure}
    
    \begin{figure}[h]
        \centering
        \includegraphics[width=\linewidth]{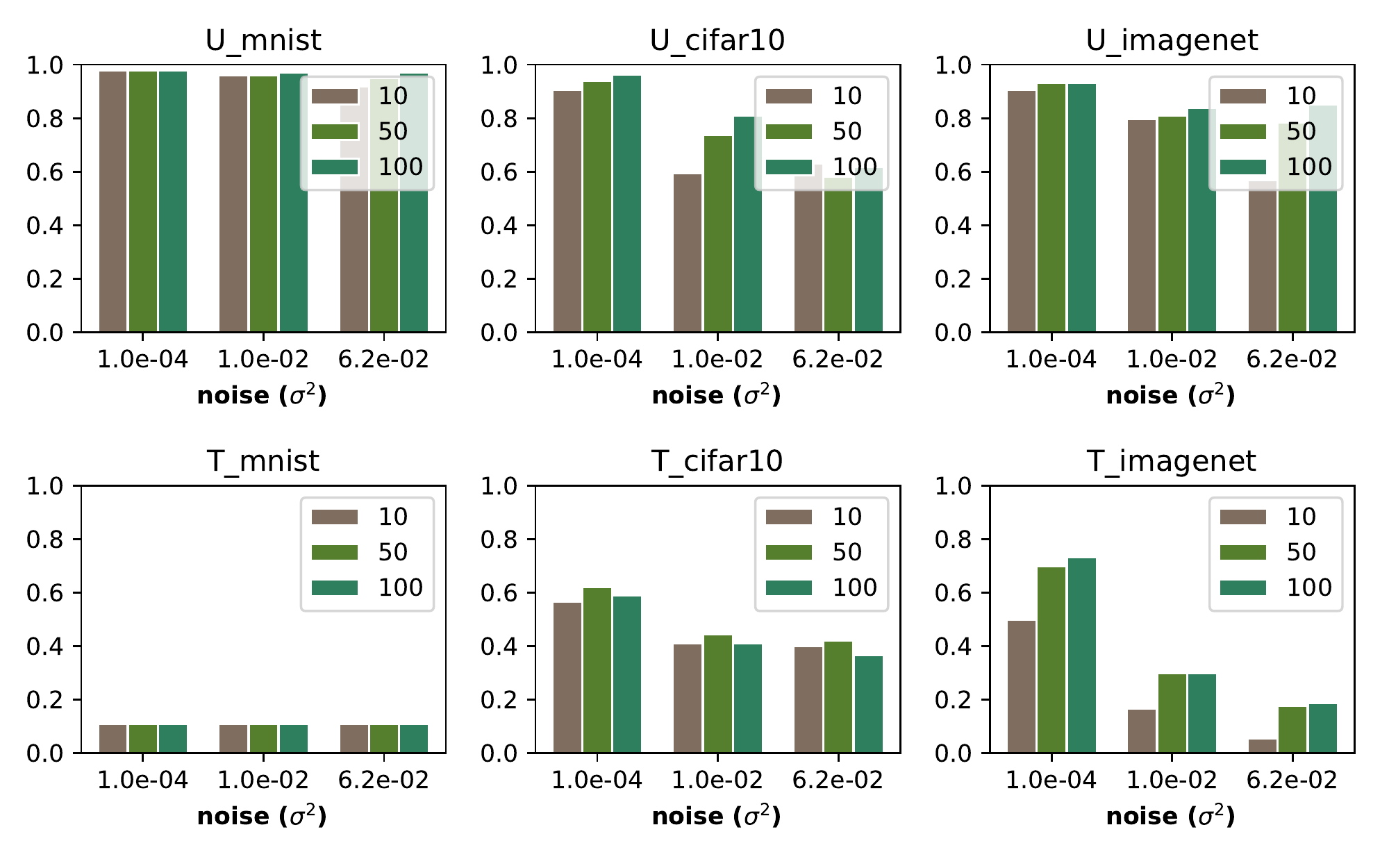}
        \caption{CW Attack success rate vs variance (groups) for adaptive attacker with increasing averaging (10, 50, 100 samples). Top row shows untargeted and bottom row shows targeted attacks. Averaging allows the white box adaptive attacker to overcome output randomization.}
        \label{fig:white_asr_samps}
    \end{figure}

As a sanity check, we also measure the attack success rate of a white box attacker (Carlini \& Wagner L2~\cite{Carlini2017} attack) with randomized output. We find that the defense has some success in defending against this attack, however the adaptive white box attacker is able to overcome output randomization by averaging over a small number of samples. This is summarized in Figures \ref{fig:white_asr_noise} and \ref{fig:white_asr_samps}.

\subsubsection{Varying step size h}
A small step size $h$ increases the finite difference estimate's sensitivity to small random perturbation. We repeated our experiments against an undefended and defended (output randomization with noise $\sigma^2=0.01$) using the untargeted ZOO attack to test its behavior when the step size increases. The results are summarized in Figure \ref{fig:varyh}. We found that increasing $h$ slightly increases the attack's success rate against output randomization, however it significantly drops the attack's success rate on the undefended network. 

We believe that because the attacker is constrained to small $h$ values, a small $\sigma^2$ is still effective against finite difference attacks due to its sensitivity to small perturbation. This sensitivity is reflected in the loss function's values during an attack.

\begin{figure}[h]
    \centering
    \includegraphics[width=\linewidth]{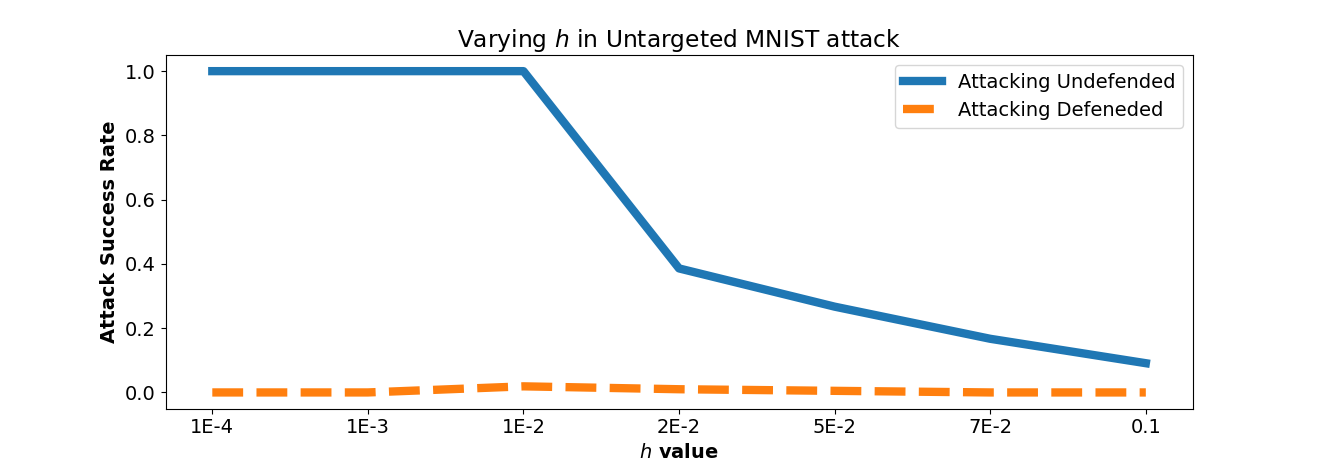}
    \caption{Untargeted ZOO attack against undefended (solid) and defended (dashed) DNNs. Both DNN's were trained on MNIST and the defended DNN uses noise $\sigma^2=0.01$. Increasing $h$ significantly drops the attack's success rate.}
    \label{fig:varyh}
\end{figure}

\begin{figure}[h]
    \centering
    \includegraphics[width=\linewidth]{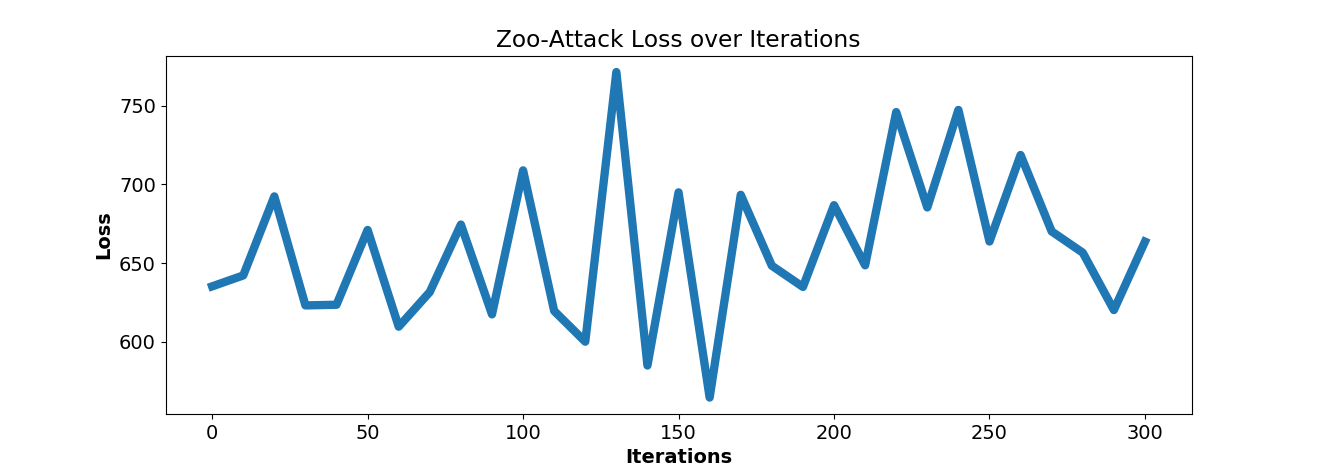}
    \caption{Loss over iterations of an untargeted ZOO attack against a DNN with output randomization (noise $\sigma^2=0.01$). The erratic behaviour leads to the attack stopping early.}
    \label{fig:zooloss}
\end{figure}

\subsubsection{Impact of output randomization on ZOO loss}
Figure \ref{fig:zooloss} shows an instance of the ZOO attack against a DNN defended with output randomization.
The loss is erratic and the attack stops because no improvement is found.
This behavior is consistent across our experiments with output randomization.

\begin{figure}[h]
    \centering
    \includegraphics[width=0.8\linewidth]{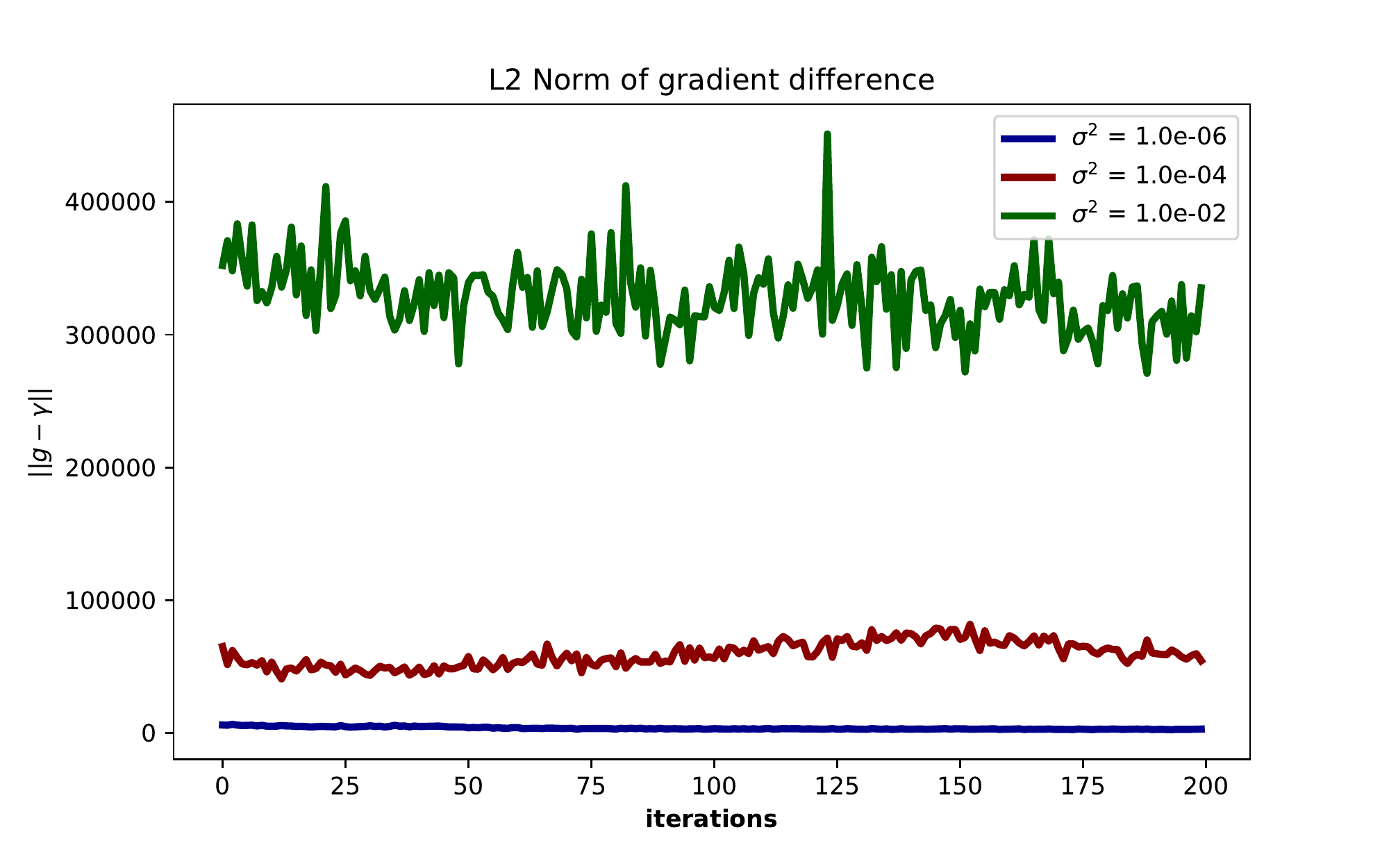}
    \caption{$L_2$ norm of difference between the true finite difference gradient calculated by ZOO on the undefended model and the finite difference gradient calculated on the defended model with increasing variance. As expected the error is significantly higher for noise with larger variance.}
    \label{fig:cifar10_grads_noise}
\end{figure}

Figure~\ref{fig:cifar10_grads_noise} shows the error between the finite difference gradients for the ZOO attack on a defended and undefended model at varying noise levels. As we expect, increased noise levels cause the overall error (measured by the norm of difference of the gradients) to increase dramatically.


\begin{table*}[h]
  \caption{ZOO black box attack success rate vs three defenses}
  \label{tab:distill}
  \centering
  \begin{tabular}{llccc}
    Dataset & Ex. Type & Distillation\cite{Papernot2016b} & Mitigation\cite{Xie2017} & OR (ours) \\
    \midrule
    MNIST & Targeted & 1.00 & - & \textbf{0.00}  \\
    MNIST & Untargeted & 0.99 & - & \textbf{0.01} \\
    CIFAR10 & Targeted & 1.00 & - & \textbf{0.011} \\
    CIFAR10 & Untargeted & 1.00 & - & \textbf{0.19} \\
    ImageNet & Untargeted & - & 0.76 & \textbf{0.005} \\
    \bottomrule
  \end{tabular}
\end{table*}

\subsubsection{Comparison to other defenses}
We evaluated defensive distillation \cite{Papernot2016b} and input randomization \cite{Xie2017} against ZOO and found that these defenses did not reduce the attack success rate significantly, as shown in Table~\ref{tab:distill}. Input randomization \cite{Xie2017} has limited success in defending against ZOO (reducing attack success rate to 0.76), however it is not as effective as output randomization. This is because randomization applied to the input is not guaranteed to affect the finite difference gradient estimates. Additionally, input randomization also does not allow fine control over model accuracy as with our proposed method. We evaluated self-ensembling \cite{Liu2017} and found that it is more susceptible to finite difference based black box attacks and has a greater robustness-accuracy trade-off compared to output randomization. Our proposed method performs better against finite difference attacks with less tunable parameters, making it more efficient to train and deploy. Results are summarized in Figures \ref{fig:comparison} and \ref{fig:comparison2}. Our evaluation of the test set accuracy of self-ensembling is consistent with that of the original authors.

\begin{figure}
    \centering
    \includegraphics[width=.799\linewidth]{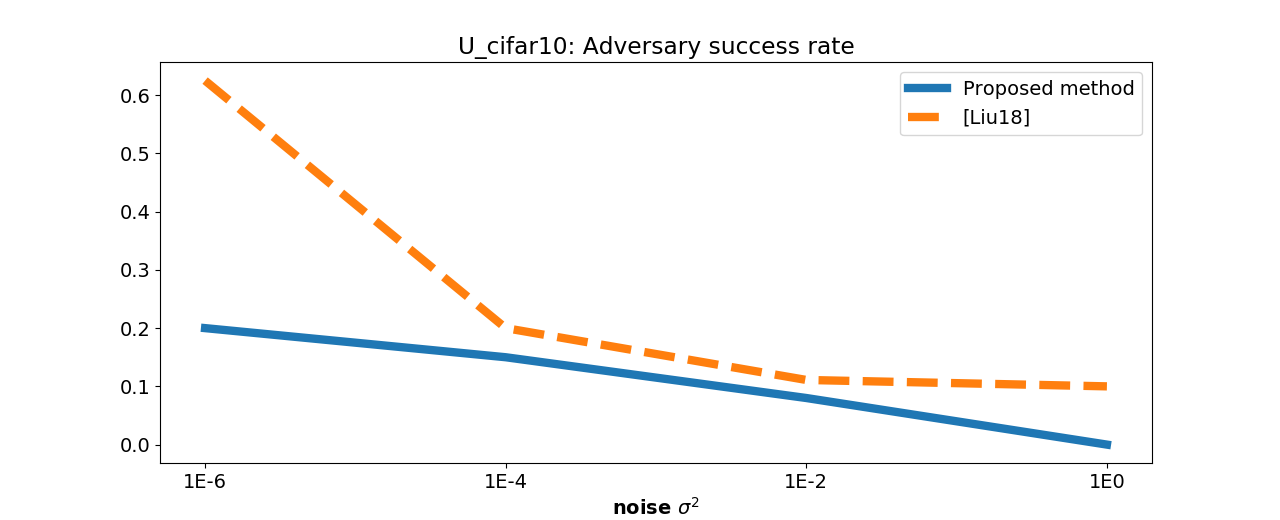}
    \caption{Untargeted ZOO attack (CIFAR10) success rate against output randomization (solid) and self-ensembling (dashed) vs variance for the non-adaptive attacker. Even at the highest $\sigma^2$, output randomization is more effective at thwarting the attack. }
    \label{fig:comparison}
\end{figure}

\begin{figure}
    \centering
    \includegraphics[width=.799\linewidth]{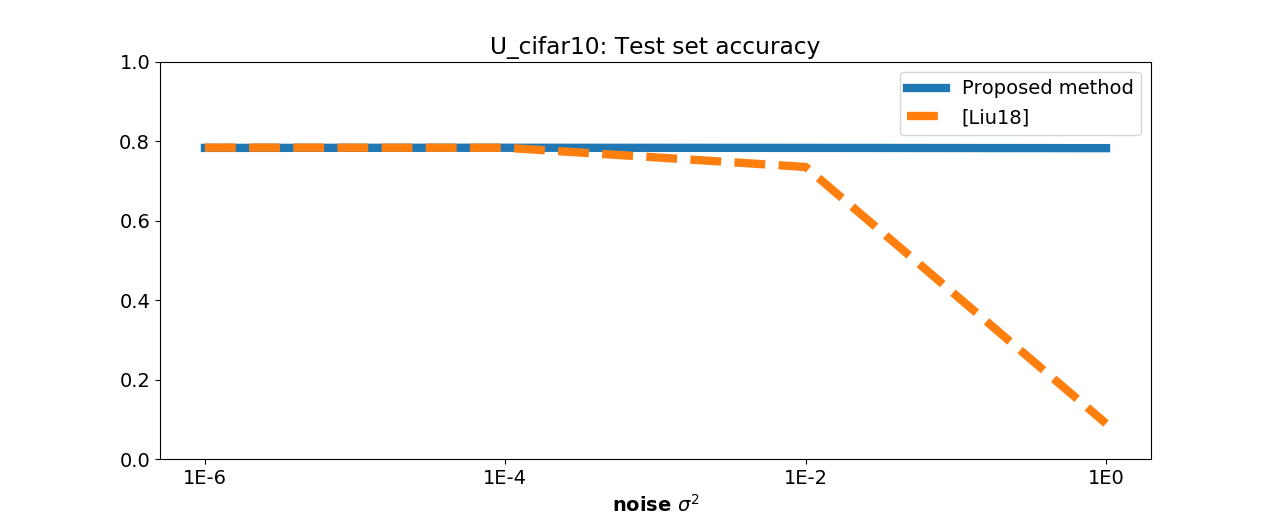}
    \caption{Test set classification accuracy by a DNN (CIFAR10) with output randomization vs self-ensembling with respect to $\sigma^2$. Increasing $\sigma^2$ causes self-ensembling to trade accuracy for robustness at a higher rate than output randomization.}
    \label{fig:comparison2}
\end{figure}

\subsubsection{Adaptive attack results}
    \begin{figure}
        \centering
        \includegraphics[width=\linewidth]{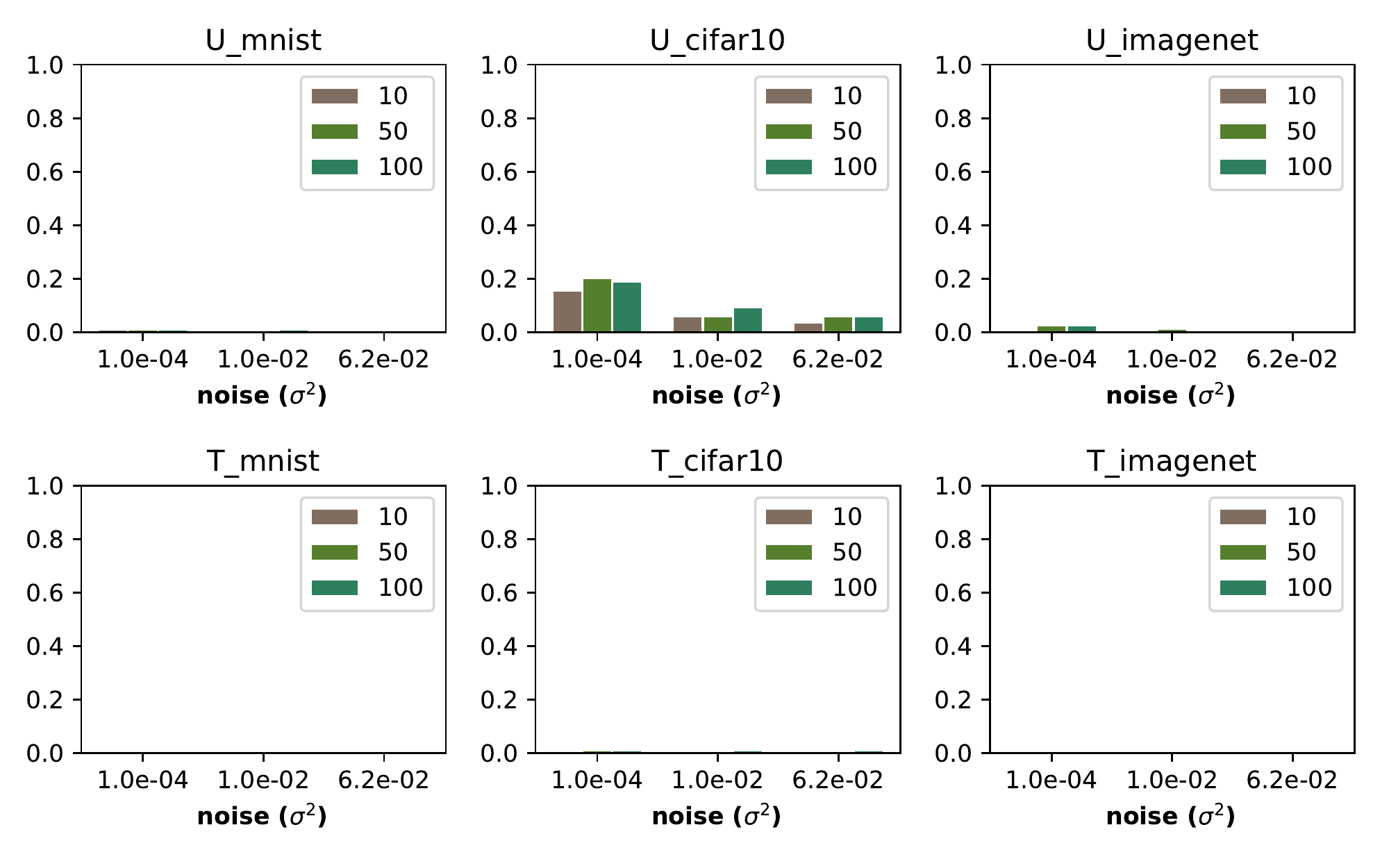}
        \caption{ZOO-Attack success rate vs variance (bar groups) for adaptive attacker with increasing averaging (10, 50, 100 samples). Top row shows untargeted and bottom row shows targeted attacks. Averaging does not improve attack success rate of the black box attacker.}
        \label{fig:black_asr_samps}
    \end{figure}
    
Figure \ref{fig:black_asr_samps} was omitted from the main manuscript due to space. However, we describe the shown results as well as discuss how the attack can be altered to be made successful. We believe that increasing the allowed number of samples given to the adversary (to average over) would increase the rate of success. However, we were unable to test this given our own computational limits.

\subsection{Empirical results: output randomization training}
In this section, we present supplemental empirical results for the effects of output randomization training on the robustness of deep neural networks.

    \subsection{Varying $\sigma$ at training time}
    We experimented with different values for standard deviation $\sigma$ for the Gaussian from which we sample noise to be added to the pre-softmax layer of the target models. The results of this experimentation are shown in Figure \ref{fig:varysigma}
    
    \begin{figure}
        \centering
        \includegraphics[width=\linewidth]{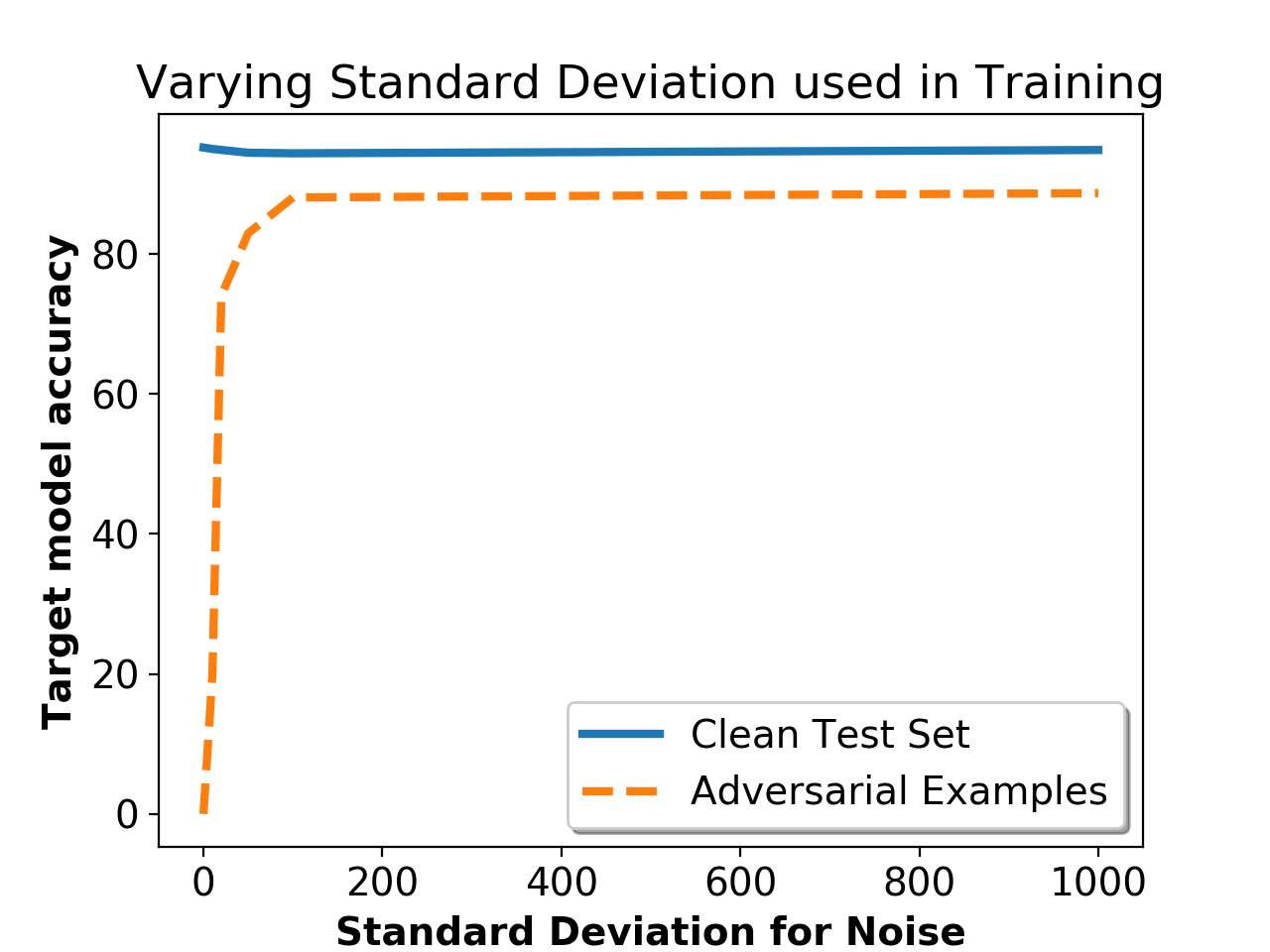}
        \caption{Adversarial robustness and clean test set accuracy of Wide Resnet 32-10 with variable $\sigma$ for Gaussian noise added to pre-softmax layer.}
        \label{fig:varysigma}
    \end{figure}
    
    We were very surprised to see that there is very little trade-off between adversarial robustness and clean test set accuracy as the value of $\sigma$ increases. Even in our experimentation with adding noise to the pre-softmax layer in our black box defense, we saw a substantial trade-off between test set accuracy and robustness with $\sigma \ge 1e3$. It is expected that larger amounts of noise must be added to the pre-softmax output because the softmax function will scale the probabilities, however it was unexpected how much noise could be added using Wide Resnet 32-10 before seeing instability at training time. We hypothesize that this is due to Wide Resnet 32-10 being a higher capacity network \cite{Madry2017}, especially when compared to the 3-convolutional layer models used to evaluate the ZOO attack. 
    
    However, we do see extremely large $\sigma$ values leading to low accuracies for our target model. This is due to loss of precision with the softmax function. With large $\sigma$ values, the softmax function may return a binary indicator vector for the predicted label due to computational instability. Consider a simple example with logits $l = $ [0, 1e10]. As part of the softmax function, we compute $e^{l - \mbox{max}(l)}$. Due to loss of precision, we get [0.0, 1.0], which is a binary vector indicating the predicted label. This leads to instability at test time as the model is heavily penalized for any incorrect prediction. 
    
    Furthermore, we output the unscaled logits at training time and saw that the unscaled logit representing the correct class linearly scaled with the value of $\sigma$ as the model began to converge. This further supports the claims made in Section 6 that each step of stochastic gradient descent, attempts to update the parameters of the model such that the logit of the correct class is larger than the logits of the remaining classes. For example, set $\sigma=100$. As the model begins to converge, the logit of the predicted class become an order of magnitude 2 (x100) greater than the logits of the remaining classes. This trend stayed true throughout our experiments for varying $\sigma$ at training time.

    \subsubsection{Output randomization training with models from black box experiments}
    We also repeated the white box experiments using the MNIST and CIFAR10 models from our experiments evaluating output randomization as a defense against black box finite difference attacks. The results are shown in Figures \ref{fig:whitebox_mnist} and \ref{fig:whitebox_cifar}.
    \begin{figure}[h]
        \centering
        \includegraphics[width=\linewidth]{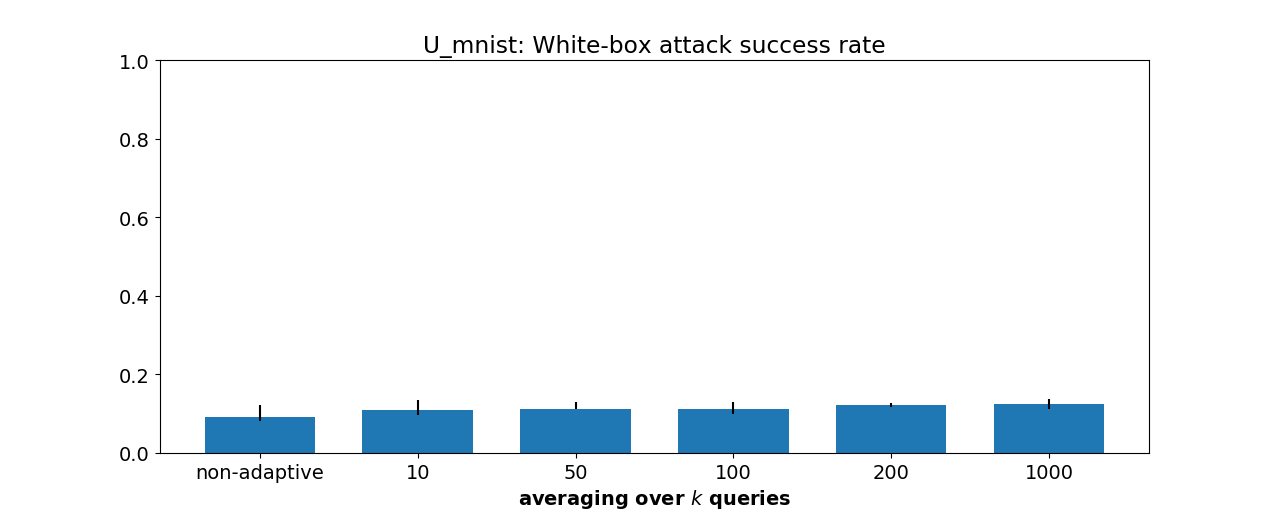}
        \caption{MNIST}
        \label{fig:whitebox_mnist}
        \end{figure}
            
    \begin{figure}[h]
        \centering
        \includegraphics[width=\linewidth]{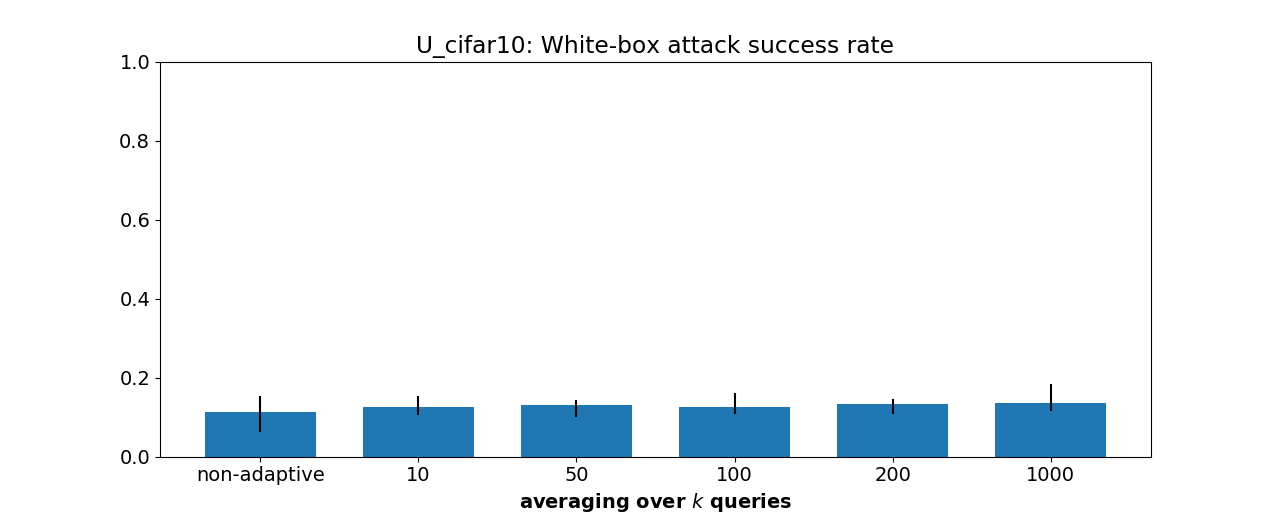}
        \caption{CIFAR10}
        \label{fig:whitebox_cifar}
        \caption{Untargted white box adaptive and non-adaptive attacks on DNNs trained on MNIST and CIFAR10 with output randomization with a noise $\sigma^2$ of 0.01. The effectiveness is similar to that seen in Figure \ref{fig:whitebox_mnist}. Training with output randomization increases resistance against adaptive white box attacks. }
    \end{figure}

    We selected $\sigma^2 = 0.01$ for the variance of the Gaussian distribution from which we sample our noise because it maintains the classifiers' test set accuracy while also being susceptible to a white box attacker, as we showed in the main body of the paper. 
    Incorporating output randomization with training significantly lowered the white box attack's success rate when attacking the DNNs compared to using output randomization alone. 
    
\begin{table*}[h]
    \caption{Noise insertion after every convolutional layer, the first convolutional layer, and the second convolutional layer in Wide Resnet 32-10 residual blocks}
    \label{tab:all_conv}
    \centering
    \begin{tabular}{llccc}
        Location & Noise $\sigma$ & Clean Accuracy & PGD (Xent) & PGD (CW) \\
        \midrule
        Every convolutional layer & 5  & 86.92\% & 0.0\% & 0.0\% \\
        & 10 & 66.36\% & 0.09\% & 0.14\% \\
        & 15 & 59.27\% & 0.13\% & 0.20\% \\
        \midrule
        First convolutional layer & 10   & 89.5\%  & 0.0\% & 0.0\% \\
        & 20   & 83.27\% & 0.0\% & 0.0\% \\
        & 50   & 73.4\%  & 0.0\% & 0.0\% \\
        & 100  & 69.42\% & 0.0\% & 0.02\% \\
        & 500  & 55.32\% & 0.16\% & 0.14\% \\
        & 1000 & 50.44\% & 0.17\% & 0.20\% \\
        \midrule
        Second convolutional layer & 10   & 91.08\%  & 0.0\% & 0.0\% \\
        & 20   & 87.14\% & 0.0\% & 0.0\% \\
        & 50   & 75.59\%  & 0.02\% & 0.02\% \\
        & 100  & 66.65\% & 0.05\% & 0.03\% \\
        & 500  & 35.96\% & 0.16\% & 0.15\% \\
        & 1000 & 9.71\% & 9.61\% & 10.23\% \\      
        \bottomrule
    \end{tabular}
\end{table*}

\subsubsection{Defending against Transfer Attacks}
We found the results of our experimentation with output randomization training against black box transfer attacks to be interesting.

\begin{figure}
    \centering
    \includegraphics[width=0.85\linewidth]{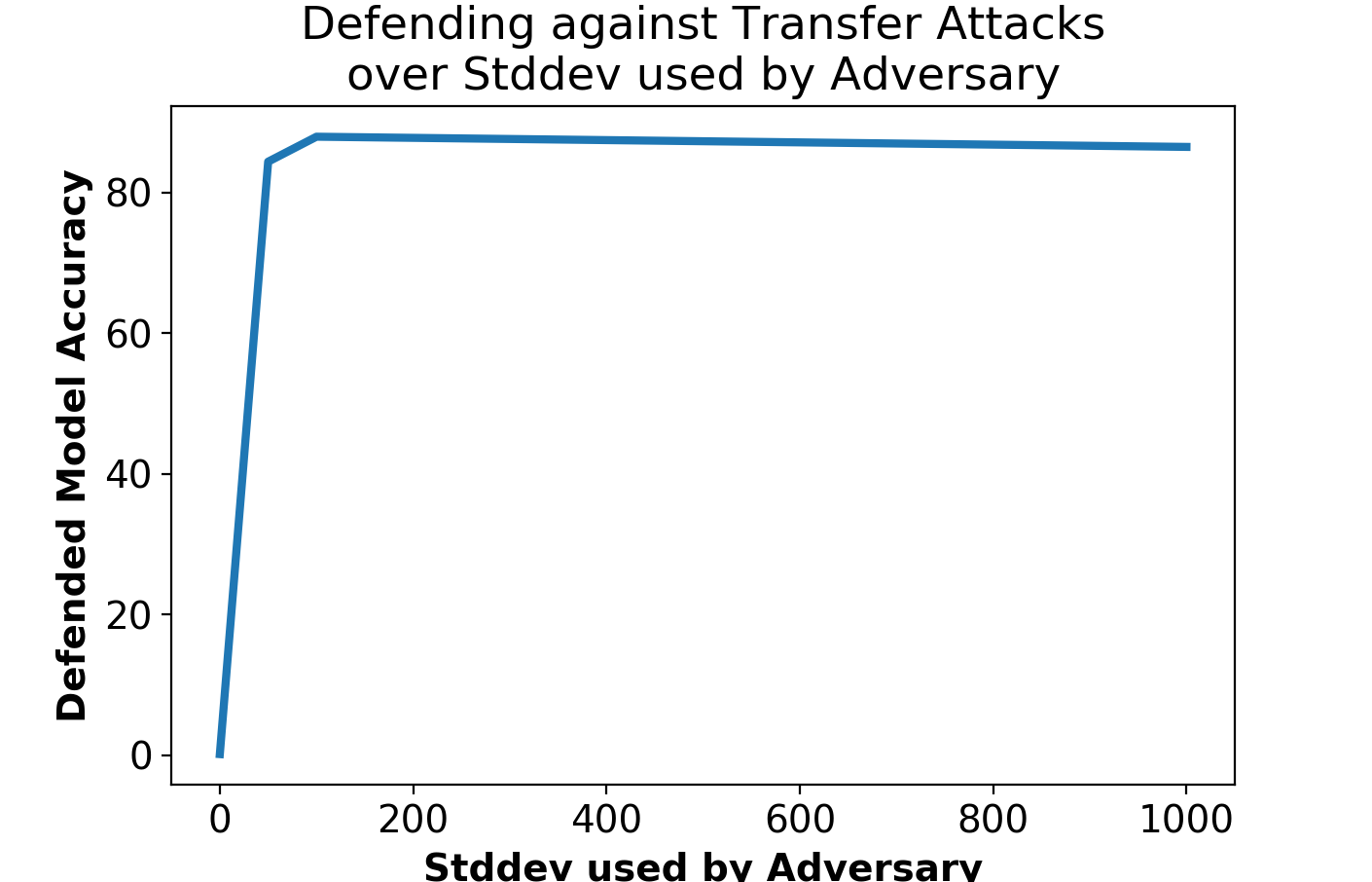}
    \caption{Wide Resnet 32-10 robustness against varying transfer attacks when trained with output randomization $\sigma=100$. The adversary has knowledge of the training data and target model's architecture.}
    \label{fig:transfer}
\end{figure}

In Figure \ref{fig:transfer}, we show the robustness of a CIFAR-10 model defended with output randomization training with $\sigma=100$ against multiple transfer attacks. We assumed a very powerful adversary that has the exact architecture and training set used by the defended model. The transfer attack (1) trains a substitute model using the training set (2) attacks the substitute model and (3) transfers successful adversarial examples to the target model. We repeated the experiment varying $\sigma$ used by the adversary.

Interestingly, an adversary is very successful when training a substitute model with no noise. This further supports our hypothesis in the main paper that an adversary becomes "stuck" in a local maximum when attacking the target model. We can see that the transfer attack's success rate drastically falls when the adversary attempts to mimic output randomization training. However, this means that output randomization can be easily thwarted by transfer attacks. Thus, we suggest combining output randomization and adversarial training to defend against transfer attacks, similar to defending against PGD(CW) attacks.

\subsubsection{Experimenting with noise location}
Other defenses such as \cite{Madry2017}, \cite{Xie2017}, and \cite{Liu2017} use randomization and noise in different locations in a model's architecture. Although our proposed defense randomizes the model's output by injecting noise to the final layer, we experimented with different locations for our noise injection. Specifically, we injected noise after every convolutional layer, after the first convolutional layer, and after the second convolutional layer in each residual block in Wide Resnet 32-10. The results are shown in Table \ref{tab:all_conv}.

As stated in \cite{Liu2017}, we found that noise needed to be injected at test time in addition to the training phase to maintain high classification accuracy with $\sigma=10$. However, we found in all three of our experiments that the clean test set accuracy dropped rapidly as $\sigma$ increased. In our experimentation with different $\sigma$ values, we see that output randomization training outperforms randomization in other locations in terms of test set accuracy and adversarial robustness.

\end{document}